%% file: adol_arxiv.tex
\documentclass[a4paper,12pt]{article}
\usepackage[cm]{fullpage}

\title{Distributed Online Learning for Joint Regret with
	Communication Constraints}

\usepackage{times}
\usepackage[utf8]{inputenc} 
\usepackage[T1]{fontenc}    
\usepackage[hidelinks]{hyperref}       
\usepackage{url}            
\usepackage{amsfonts}       
\usepackage{nicefrac}       
\usepackage{microtype}      
\usepackage{xcolor}         
\usepackage{graphicx}

\usepackage{amsmath}
\usepackage{amsthm}
\usepackage{amssymb}
\usepackage[algo2e,ruled]{algorithm2e}
\usepackage{tikz,calc}
\usepackage{thmtools,thm-restate}
\usepackage{xifthen}
\usepackage{natbib}
\usepackage{subcaption}

\input{commands.tex}

\usetikzlibrary{calc}
\usetikzlibrary{positioning}
\usetikzlibrary{decorations.text}
\usetikzlibrary{decorations.pathmorphing}
\usetikzlibrary{fit}
\usetikzlibrary{shapes.geometric}
\tikzstyle{every picture}+=[remember picture]
\pgfdeclarelayer{bg}    
\pgfsetlayers{bg,main}  
\tikzstyle{agent}=[circle,draw,fill=blue!70,minimum size=0.30cm,inner sep=0pt]
\tikzstyle{agent2}=[agent,fill=gray!30]

\author{%
 {Dirk van der Hoeven} \textit{dirk@dirkvanderhoeven.com}\\
 Universit\`a degli Studi di Milano
 \and
 {H\'edi Hadiji} \textit{hedi.hadiji@gmail.com}\\
 University of Amsterdam
 \and 
 {Tim van Erven} \textit{tim@timvanerven.nl}\\
 University of Amsterdam
}

\begin{document}

\maketitle

\begin{abstract}%
We consider distributed online learning for joint regret with
communication constraints. In this setting, there are multiple agents
that are connected in a graph. Each round, an adversary first
activates one of the agents to issue a prediction and provides a
corresponding gradient, and then the agents are allowed to send a
$b$-bit message to their neighbors in the graph. All agents cooperate
to control the joint regret, which is the sum of the losses of the
activated agents minus the losses evaluated at the best fixed common
comparator parameters $\u$. We observe that it is suboptimal for
agents to wait for
gradients that take too long to arrive. Instead, the graph should
be partitioned into local clusters that communicate among
themselves. 
Our main result is a new method that can adapt to the optimal graph
partition for the adversarial activations and gradients, where the
graph partition is selected from a set of candidate partitions. A
crucial building block along the way is a new algorithm for online
convex optimization with delayed gradient information that is
comparator-adaptive, meaning that its joint regret scales with the
norm of the comparator $\|\u\|$. We further provide near-optimal
gradient compression schemes depending on the ratio of $b$ and the
dimension times the diameter of the graph.
\end{abstract}

\section{Introduction}

We consider decentralized \emph{online convex optimization} (OCO)
with multiple agents that share information across a network to
improve the prediction quality of the network as a whole. Our motivation
comes from cases where local computation is cheap, but communication is
relatively expensive. This is the case, for instance, in sensor
networks, where the energy cost of wireless communication is typically
the main bottleneck, and long-distance communication requires much more
energy than communication between nearby sensors \citep{RabatNowak2004}.
It also applies to cases where communication is relatively slow compared
to the volume of prediction requests that each agent must serve. For
instance, in climate informatics communication may be slow because
agents are geographically spread out
\citep{mcquade2012global,mcquade2017spatiotemporal}, and in finance or
online advertising the rate of prediction requests may be so high that
communication is slow by comparison. To model such scenarios, we limit
communication in two ways: first, agents can only directly communicate
to their neighbors in a graph $\Graph$ and, second, the messages that
the agents can send are limited to contain at most $b$ bits. We further
assume that learning is fully decentralized, so there is no central
coordinating agent as in federated learning \citep{kairouz2019advances},
and no single agent that dictates the predictions for all other agents
as in distributed online optimization for consensus problems
\citep{HosseiniEtAl2013,YanEtAl2013}.

To fix the setting, assume there are $N$ agents, which are cooperating
to make sequential predictions over the course of $T$ rounds. In every
round $t$, first one of the agents $I_t$ is activated by an
adversary to select a prediction $\w_t$ from a closed and convex domain
$\domainw \subseteq \reals^d$. Then this agent receives feedback from
the adversary in the form of the (sub)gradient $\g_t = \nabla
\ell_t(\w_t)$ of a convex loss function $\ell_t$ over~$\domainw$, with
bounded Euclidean norm $\|\g_t\|\leq G$. Finally, all agents are allowed
to communicate by sending a $b$-bit message to their neighbors in
$\Graph$, and the round ends. The common goal of the agents is to
control the \emph{joint regret} with respect to comparator parameters
$\u \in \domainw$:
\begin{equation*}
	\regret_T(\u) = \sumT\left(\ell_t(\w_t) - \ell_t(\u)\right).
\end{equation*}
We refer to this setting, as \emph{\underline{d}istributed
	\underline{o}nline \underline{c}onvex \underline{o}ptimization for
	\underline{j}oint regret with \underline{c}ommunication constraints}
(DOCO-JC). Apart from the communication limit $b$, the crucial
distinction between DOCO-JC and standard OCO
\citep{shalev2011online,hazan2016introduction} is that information about
the gradients $\g_t$ takes time to travel through the graph, so the
agents suffer from delayed feedback
\citep{mcmahan2014delay,joulani2016delay,hsieh2020multiagent}. This observation has
prompted \citet{hsieh2020multiagent} to consider a more abstract
framework, in which there is no explicit graph, but only assumptions
about the delays. For instance, if $\tau$ is the maximum delay before
$\g_t$ is known by every agent, then Corollary~2 of
\citeauthor{hsieh2020multiagent} implies a joint regret bound of
\begin{equation}\label{eqn:jointclusters}
	\regret_T(\u) = \O\big(\max_{\u \in \domainw} \|\u\| \sqrt{\tau T}\big).
\end{equation}
In our setting, $\tau$ corresponds to the maximum graph distance between
any two agents that are ever active, i.e.\ the diameter $D(\Graph)$ of
$\Graph$ if all agents are activated at least once.

Although modeling only delays is an elegant abstraction, we argue that
it is ultimately insufficient and that the graph structure should be
explicitly taken into account. To see this, consider the graph $\Graph$
from Figure~\ref{fig:galaxiesFarFarAway}. In this graph, there are two
clusters of agents that are very far apart. For simplicity, suppose that
only agents from the two clusters are ever active, while the agents that
connect the clusters only serve to pass on information. Then the maximum
delay $\tau$ can be made arbitrarily large by extending the line that
connects the two clusters. There exists a much better strategy, however,
which is to have the two clusters operate independently with a maximum
delay of $\tau_j = 2$ within each cluster $j=1,2$, leading to joint
regret
\begin{equation}\label{eqn:separateclusters}
	\regret_T(\u)
	= \O\big(\max_{\u \in \domainw} \|\u\| \sqrt{\tau_1T} + \max_{\u \in \domainw} \|\u\| \sqrt{\tau_2 T}\, \big) 
	= \O\big(\max_{\u \in \domainw} \|\u\| \sqrt{T}\, \big).
\end{equation}
Comparing \eqref{eqn:separateclusters} to \eqref{eqn:jointclusters} for
arbitrarily large $\tau$, we see that explicitly taking the graph
structure into account can lead to an arbitrarily large improvement over
modeling only delays. The takeaway from this example is that it is
better for an agent to ignore information when it has to wait too long
to receive it. The same conclusion still holds even if we replace $\tau$
by more refined measures of delay.
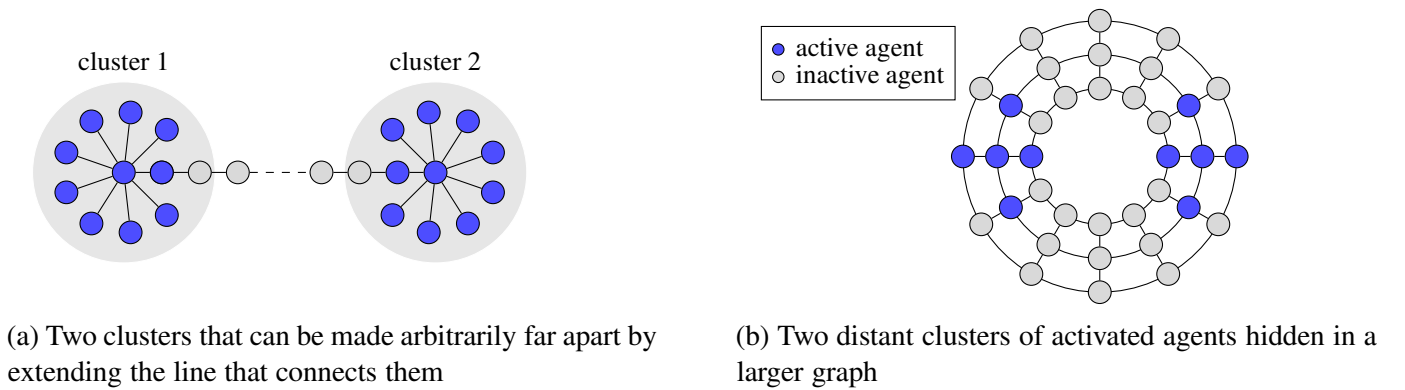
\begin{figure}[htb]
	\subcaptionbox{Two clusters that can be made arbitrarily far apart by
	extending the line that connects them\label{fig:galaxiesFarFarAway}}{%
		\newcommand{\clustersize}{8}
		\newcommand{\radius}{0.8cm}
		\newcommand{\linesp}{0.5cm}
		\centering
		\makebox{%
			\begin{minipage}[b]{0.45\textwidth}
				\begin{tikzpicture}
					\node[agent] (center1) at (0cm,0) {};
					\foreach \c in {1,...,4}{
						\ifthenelse{\c<3}{%
							\node[agent] (center\c) at ({(\c-2)*\linesp},0) {};
						}{
							\node[agent2] (center\c) at ({(\c-2)*\linesp},0) {};
						}
					}
					\foreach \c in {5,...,8}{
						\ifthenelse{\c>6}{%
							\node[agent] (center\c) at ({(\c-0.8)*\linesp},0) {};
						}{
							\node[agent2] (center\c) at ({(\c-0.8)*\linesp},0) {};
						}
					}
					\foreach \c [count=\cc from 2] in {1,...,3}
					\draw (center\c)--(center\cc);
					\foreach \c [count=\cc from 6] in {5,...,7}
					\draw (center\c)--(center\cc);
					\draw[dashed] (center4)--(center5);
					\foreach \n in {1,...,\clustersize}{
						\node[agent] at ($ (center1)
						+ ({45 + (\n-1)*270/(\clustersize-1)}:\radius) $) (n\n) {};
						\draw (center1)--(n\n);
					}
					\foreach \n in {1,...,\clustersize}{
						\node[agent] at ($ (center8)
						+ ({-135 + (\n-1)*270/(\clustersize-1)}:\radius) $) (nn\n) {};
						\draw (center8)--(nn\n);
					}
					\begin{pgfonlayer}{bg}    
						\node [draw=white,fill=gray!20,circle,minimum
						size={3*\radius},label=above:\footnotesize{cluster 1},draw] at (center1) {};
						\node [draw=white,fill=gray!20,circle,minimum size={3*\radius},label=above:\footnotesize{cluster 2},draw] at (center8) {};
					\end{pgfonlayer}
				\end{tikzpicture}\\
				\vspace{0.03cm}
			\end{minipage}
	}}
	\hspace{0.05\textwidth}
	\subcaptionbox{Two distant clusters of activated agents hidden in a larger graph\label{fig:embeddedGalaxies}}{%
		\newcommand{\nragents}{12}
		\newcommand{\radius}{0.9cm}
		\newcommand{\rspace}{0.45cm}
		\makebox[0.45\textwidth][l]{%
			\begin{tikzpicture}
				\foreach \n in {1,...,\nragents}{
					\ifthenelse{\n=1 \OR \n=7}{%
						\node[agent] at ({(\n-1)*360/\nragents}:\radius) (agent\n-1) {};
						\node[agent] at ({(\n-1)*360/\nragents}:{\radius+2*\rspace}) (agent\n-3) {};
					}{
						\node[agent,fill=gray!30] at ({(\n-1)*360/\nragents}:\radius) (agent\n-1) {};
						\node[agent,fill=gray!30] at ({(\n-1)*360/\nragents}:{\radius+2*\rspace}) (agent\n-3) {};
					}
					\ifthenelse{\n<3 \OR \n=6 \OR \n=7 \OR \n=8 \OR \n=\nragents}{%
						\node[agent] at ({(\n-1)*360/\nragents}:{\radius+\rspace}) (agent\n-2) {};
					}{
						\node[agent,fill=gray!30] at ({(\n-1)*360/\nragents}:{\radius+\rspace}) (agent\n-2) {};
					}
					\draw (agent\n-1) -- (agent\n-2) -- (agent\n-3);
				}
				\begin{pgfonlayer}{bg}    
					\node[circle,draw,minimum size=2*\radius] at (0,0) {};
					\node[circle,draw,minimum size=2*(\radius+\rspace)] at (0,0) {};
					\node[circle,draw,minimum size=2*(\radius+2*\rspace)] at (0,0) {};
				\end{pgfonlayer}
				\node[agent,minimum size=0.15cm,above
                                left=0.35cm and 2.5cm of agent6-3] (activeagent) {};
				\node[right=0cm of activeagent]
                                (activelabel) {\footnotesize{active agent}};
				\node[agent2,minimum size=0.15cm,below=0.2cm of activeagent] (inactiveagent) {};
				\node[right=0cm of inactiveagent]
                                (inactivelabel) {\footnotesize{inactive agent}};
				\node[fit=(activeagent) (activelabel) (inactiveagent)
				(inactivelabel),draw,inner ysep=0pt,inner xsep=2pt,xshift=-2pt] {};
	\end{tikzpicture}}}
	\caption{Two clusters far apart}
\end{figure}
Unfortunately, partitioning $\Graph$ into subgraphs that exchange
information is not always as easy as in
Figure~\ref{fig:galaxiesFarFarAway}, because the clusters may be hidden
in a larger graph (see Figure~\ref{fig:embeddedGalaxies}) and the
optimal partition may depend on the adversarial activations $I_t$, and
also on the gradients $\g_t$ and the number of bits $b$ that are allowed
for communication. We therefore introduce a method that can learn the
optimal partition from a set of candidate partitions. Formally, let
$\Qset$ be a collection of subgraphs of~$\Graph$, which will be the
building blocks for the candidate partitions. Then by a
\emph{$\Qset$-partition} of the active agents we mean a disjoint
collection $\{\F_1,\ldots,\F_r\}$ of elements $\F_j \in \Qset$ such that
every node in $\Graph$ that is ever activated during any of the $T$
rounds, is contained in one of the $\F_j$. The size $r$ may vary between
$\Qset$-partitions. We show, in Theorem~\ref{th:partidet}, that we can
adapt to the best partition of the active agents at a cost that scales
logarithmically with the size of~$\Qset$: 
\begin{multline}\label{eqn:partitionlearnDeterministic}
	\sum_{j=1}^r \regret_{\F_j}(\u_j) = \O\Big(\sum_{j=1}^r
	\|\u_j\|\Big(\sqrt{D(\F_j) T_j\ln \big(1+|\Qset| D(\F_j) \|\u_j\|T_j\big)}
	\Big)+ \text{communication cost}\Big)\\
	\text{for any $\Qset$-partition $\{\F_1,\ldots,\F_r\}$ and any
		$\u_1,\ldots,\u_r \in \domainw$,}
\end{multline}
where $\regret_{\F_j}(\u_j) = \sum_{t:I_t \in \F_j} (\ell_t(\w_t) -
\ell_t(\u_j)$ is the joint regret measured for the $T_j$ rounds in which
the active agent $I_t$ is a node in $\F_j$. Up to the logarithmic
factor, this bound implies \eqref{eqn:separateclusters} with the maximum
delays $\tau_j$ replaced by the diameters $D(\F_j)$ of the partition
cells $\F_j$, as is natural in our setting. In fact, there are two
further improvements: first, the comparators $\u_j$ may differ between subgraphs $\F_j$, which makes the procedure more
robust against heterogeneous environments and sensor malfunctions. And,
second, we do not place any restriction on the size of the
domain~$\domainw$, but instead we automatically adapt to the unknown
comparator norms $\|\u_j\|$. In fact, comparator-adaptivity is crucial
to our approach: it allows aggregating over different delays in
receiving the gradients using the iterate addition trick by
\citet{cutkosky2019combining}. This would not be possible using existing
aggregation methods for prediction with expert advice with delayed
gradients, which would all incur an overhead growing with the largest
possible gradient delay. Regarding the logarithmic factors, it is known
that a factor $\ln(1 +
\|\u_j\|T_j)$ is unavoidable for comparator-adaptive algorithms
\citep{orabona2013dimension}. We discuss the logarithmic dependence on
$|\Qset|$ further below.

Our approach is based on having the agents communicate compressed
approximations $\hat{\g}_t$ of the gradients $\g_t$, which are forwarded
through the network for at most $D_\Qset = \max_{\F \in \Qset} D(\F)
\leq D(\Graph)$ rounds. 
This means that nodes may need to forward up to $D_\Qset$ compressed
gradients at the same time, leaving $\floor{b/D_\Qset}$ bits per
gradient, and thus the communication cost grows with $D_\Qset$. Approximations $\hat{\g}_t$ may either be deterministic or stochastic,
depending on whether the encoder that produces them is allowed to
randomize. The method that achieves
\eqref{eqn:partitionlearnDeterministic} uses a deterministic encoding
scheme, for which 
\begin{equation}\label{eqn:comcost}
	\text{communication cost} = \O\Big(2^{-b/(d D_{\Qset})}\sum_{j = 1}^r\|\u_j\|T_j\Big).
\end{equation}
We see that we need roughly $b = \Omega(d D_\Qset \ln T)$ bits to be
sure that the
communication cost is under control. In contrast,
if we allow for stochastic encodings, then the expected
communication cost can be reduced further. As shown in
Theorem~\ref{th:partistoch}, it is possible to obtain the following bound, provided
that $b \geq D_\Qset \, ( 3\lceil \log_2(d) \rceil+ 2)$: 
\begin{multline}\label{eqn:partitionlearnStochastic}
	\E\!\bigg[\sum_{j=1}^r \regret_{\F_j}(\u_j)\bigg] = \O\bigg(\sum_{j=1}^r
	\|\u_j\|G\Big(
	\sqrt{\Big(1+\frac{d D_\Qset}{b}\Big) D(\F_j) T_j\ln \big(1 + |\Qset|
		D(\F_j) \|\u_j\| T_j G\big)}\Big)\bigg)\\
	\text{for any $\Qset$-partition $\{\F_1,\ldots,\F_r\}$ and any
		$\u_1,\ldots,\u_r \in \domainw$.}
\end{multline}
Comparing \eqref{eqn:partitionlearnStochastic} to
\eqref{eqn:partitionlearnDeterministic}+\eqref{eqn:comcost}, we now
obtain the same rate as soon as $b = \Theta(d D_\Qset)$, gaining an $\ln
T$ factor. And, more importantly, whereas the deterministic
communication cost in \eqref{eqn:comcost} can be linear in $T$ for $b = o(d
D_\Qset \ln T)$,
the stochastic encoding result in
\eqref{eqn:partitionlearnStochastic} allows for a number of bits~$b$
that is sublinear in $d D_\Qset$, at the cost of (only) a worse constant
factor in the bound. This makes it possible to choose a trade-off between
communication cost and joint regret performance. 

\paragraph{Approach and Organization of the Paper}

As mentioned, our approach aggregates multiple comparator-adaptive
subalgorithms that each incur their own maximum delay. Since existing
comparator-adaptive algorithms are not suited for compressed or delayed
gradients, we introduce a new comparator-adaptive algorithm for the
DOCO-JC setting in Section~\ref{sec:compgraph}. As discussed below
Lemma~\ref{lem:potentialround}, the key to its development is a novel
inequality that generalizes the so-called prod bound \citep[Lemma
2.4]{cesa2006}. Since the prod bound is at the core of many adaptive
algorithms in the literature, for example the algorithms in
\citep{koolen2015second, vanErven2017metagrad, cutkosky2018black,
	wang2019adaptivity, vanErven2021metagrad}, our new inequality may also be useful to develop
other adaptive algorithms for settings with delayed gradients. For
compressed gradients such that $\|\hat{\g}_t - \g_t\| \leq \varepsilon$,
the new comparator-adaptive algorithm satisfies the following regret
bound:
\begin{equation*}\label{eq:regbound intro}
	\regret_T(\|\u\|) = \O\left(\|\u\|\sqrt{\Lambda_T\ln(\|\u\|\Lambda_T + 1)} + \|\u\|T\varepsilon\right), 
\end{equation*}
where $\Lambda_T = \sumT \left(\|{\g}_t\|^2 + 2\|{\g}_t\|\sum_{i \in
	\gamma(t)} \|{\g}_i\|\right)$ is a standard measure of the effect of
gradient delays called the \emph{lag}
\citep{hsieh2020multiagent,joulani2016delay,mcmahan2014delay}, with
$\gamma(t) \subseteq \{1,\ldots,t-1\}$ denoting the set of indices of
past gradients that are unavailable to the active agent $I_t$. As the
maximum delay in $\Graph$ is $D(\Graph)$, there can be at most
$|\gamma(t)| \leq D(\Graph)$ gradients that are unavailable at any time,
and consequently the lag satisfies $\Lambda_T \leq G^2(1 + 2D(\Graph))T$.
In Section~\ref{sec:limited_com} we combine the algorithm from
Section~\ref{sec:compgraph} with both deterministic and stochastic
encodings for the gradients. Our encodings are based on simple
combinations of standard covering arguments, but we prove matching lower
bounds showing that they yield guarantees that are worst-case optimal
(up to log factors), and they have the appeal of being straightforward
to analyse and implement. Finally, in
Section~\ref{sec:partitionlearning}, we obtain
\eqref{eqn:partitionlearnDeterministic} and
\eqref{eqn:partitionlearnStochastic} by aggregating multiple instances
of the methods from Section~\ref{sec:limited_com}, instantiated for
different maximum delays. Even though the algorithms from
Section~\ref{sec:limited_com} are worst-case optimal, it is not clear
whether the logarithmic dependence on $|\Qset|$ in
\eqref{eqn:partitionlearnDeterministic} and
\eqref{eqn:partitionlearnStochastic} that results when combining them,
is also optimal. We leave this as an open question for future
work. 


%
\subsection{Related Work}
There has been much work on distributed architectures. However, the
majority of the literature is about federated or parallel computation
(see \citet{kairouz2019advances} for an extensive review of the
federated setting), where multiple workers are under the supervision of
a central coordinator. In contrast, we study a decentralized setting, in
which no central authority coordinates the learning. We also study the
impact of delays and communication limits. We therefore focus our literature review on decentralized learning and on other works with communication limits. 	

\paragraph{Decentralized Online Convex Optimization} Most directly related to our setting are decentralized OCO settings, in which a set of agents in a network collectively try to optimize an objective that is revealed sequentially. This includes
the work of \citet{hsieh2020multiagent} on delay-tolerant algorithms.
The main technical
difficulty they encounter is to tune the learning rates
for a dual-averaging/follow-the-regularized-leader type approach,
which is especially challenging because of the requirement of
maintaining a non-decreasing learning rate. 
\citet{cesa2020cooperative} consider a setting where multiple nodes can be active per round. In each round all active nodes make a prediction and suffer the same loss. The most important difference with our setting is that information is not forwarded through the network, so agents only hear about the gradients of their direct neighbors in $\Graph$. The authors show that it is possible to obtain $\regret_T(\u) = \O(\sqrt{A(\Graph)T})$ and $\regret_T(\u) = \O(\sqrt{Q(\Graph)T})$ for stochastic and adversarial activations respectively, where $A(\Graph)$ is the independence number of $\Graph$ and $Q(\Graph)$ is the clique covering number. 
Finally, in \citep{cao2021decentralized}, a setting with event-triggered communication is introduced.

\paragraph{Distributed Online Optimization} 
Distributed Online Optimization is inspired by (offline) distributed
optimization \citep{DuchiEtAl2010,ScamanEtAl2018} and developed in
\citep{HosseiniEtAl2013,YanEtAl2013}. The difference with the setting we
consider is the notion of regret. In Distributed Online Optimization,
the \emph{collective regret} is analysed, in which the global loss per
round is a sum of local losses per agent, but this global loss is always
evaluated at the prediction of one of the agents. This collective regret
is closer to the distributed optimization objective used, e.g., for
wireless sensor networks \citep{RabatNowak2004}. There exist extensions
for time-varying networks with a specific structure,
\citep{mateos-nunez2014distributed, akbari2015distributed}, and there is
a version of collective regret where the comparator changes between
rounds \citep{shahrampour2018distributed, zhang2019distributed}.
\citet{hsieh2020multiagent} provide an extensive review of Distributed
Online Optimization and a reduction from collective regret to joint regret.

\paragraph{Communication-Limited Settings}	Communication can be a
performance bottleneck in distributed systems (see, e.g., a discussion
of performance in the context of parallel training of deep neural
networks in \citep{seide2014-bit}). This has generated much interest in
diverse fields for communication-constrained distributed tasks,
including in optimization \citep{alistarh2017qsgd}, for mean-estimation
\citep{suresh2017distributed}, for hypothesis testing
\citep{SzaboVuursteenZanten2020}, and for inference
\citep{acharya2020inference}. Two lines of research are closest to our
work. The first, in \citep{tang2018communication,
	koloskova2019decentralized, vogels2020powergossip} and references
therein, studies variants of Stochastic Gradient Descent (SGD) used in
decentralized optimization under bandwidth-limited gossip communication,
often with the aim of training deep neural networks. Another line of
work is devoted to online learning with communication constraints, with
lower bounds for online learning problems with communication constraints
\citep{shamir2014fundamental}, and online learning  in a serial
multi-agent framework \citep{acharya2019distributed}. Most of these
works focus on cases where the number of bits per message is at least
linear in $d$ with the exceptions of \citep{acharya2019distributed,
	mayekar2020ratq}. To our knowledge, we are the first to incorporate
communication constraints into a decentralized online learning
framework.

\paragraph{Comparator-Adaptive Algorithms} Recently a series of work has developed comparator-adaptive algorithms for various settings. 
For example, 
for standard OCO \citep{mcmahan2014unconstrained, orabona2016coin, foster2017parameter, cutkosky2017online, cutkosky2018black}, 
scale-free comparator-adaptive algorithms \citep{kotlowski2017scale, kempka2019adaptive},
comparator-adaptive algorithms with unbounded stochastic gradients \citep{jun2019parameter, vanderhoeven2019user}, 
for convex bandits \citep{van2020comparator}, 
for dynamic and strongly adaptive OCO \citep{cutkosky2020parameter}, 
or with an unknown bound on the gradients \citep{cutkosky2019artificial, mhammedi2020lipschitz}.

\subsection{Further Assumptions and Notation}
A network is an (undirected) graph $\Graph = (\mathcal N, \mathcal E)$, consisting of a set of nodes
$\mathcal N$ and edges $\mathcal E$ between them.
Throughout the paper norms are always the Euclidean norm. We assume that
the (sub)gradients are bounded by $\|\g_t\| \leq G$, and that $G$ and the
time horizon $T$ are known to the agents in advance. We do not need to
assume an oblivious adversary, because the agents only randomize when
choosing $\hat{\g}_t$, which happens after the adversary has already
revealed $\g_t$.

\paragraph{Encoding the Gradients}
Since a $b$-bit message may contain at most $D(\Graph)$ gradients, we
reserve $k = \floor{b/D(\Graph)}$ bits per gradient. After the active
node $I_t$ observes the gradient $\g_t$, it builds a $k$-bit compressed
gradient $C(\g_t) \in \{ 0, 1\}^k$ and sends it to other nodes. These
then decode to $\hat \g_t = F(C(\g_t)) \in \R^d$, which is used as an
approximation of the true gradient. We assume that it is common knowledge
among the agents at which time $t$ each compressed gradient $\hat{\g}_t$
was produced, and that agents also do not need to explicitly encode how
many gradients they are forwarding at any given time. These assumptions
can always be satisfied by adding a few extra bits of meta-information.

\section{Comparator-Adaptive Algorithm for DOCO-JC}\label{sec:compgraph}

In this section we introduce the main building block for our approach: a
comparator-adaptive algorithm that can handle both missing and approximate
gradients. With some minor modifications the algorithms in this section
can also be used in the OCO with delays setting, where one only needs to
track which gradients are available for prediction,
and not the node that made the prediction.
Without loss of generality, we only consider $\mathcal{W} =
\reals^d$, because it is straightforward to reduce constrained domains
to this case using a reduction by
\citet[Theorem~3]{cutkosky2018black}. As observed by
\citet{cutkosky2018black} comparator-adaptive algorithms can be
constructed by separately learning the direction $\u/\|\u\|$ and scale $\|\u\|$,
where learning the direction is a standard constrained learning task on the unit ball
and most of the difficulty lies in solving the unconstrained $1$-dimensional scale
problem while being adaptive to $\|\u\|$. Suppose agent $I_t$ predicts $\z_t$
for the direction, satisfying $\|\z_t\| \leq 1$, according to an
algorithm~$\mathcal{A}_\mathcal{Z}$, and it predicts $v_t \in
\reals_+$ for the
scale following an algorithm $\mathcal{A}_\mathcal{V}$. Then its
joint prediction is $\w_t = v_t \z_t$. The corresponding notions of
joint regret are $\lregret_T^\mathcal{Z}\big(\frac{\u}{\|\u\|}\big) = \sumT
\langle \z_t - \frac{\u}{\|\u\|}, \g_t \rangle$ for the direction, and
$\lregret_T^\mathcal{V}(\|\u\|) = \sumT (v_t - \|\u\|) \langle \z_t ,
\g_t \rangle$ for the scale. Then, by the black-box reduction in
Algorithm~\ref{alg:black box} in Appendix~\ref{app:black-box}, due to
\citet{cutkosky2018black}, the total joint regret of the algorithm is bounded
by 
\begin{equation*}
	\regret_T(\u) \leq \|\u\|\lregret_T^\mathcal{Z}\Big(\frac{\u}{\|\u\|}\Big) + \lregret_T^\mathcal{V}(\|\u\|).
\end{equation*}
It follows that, as long as $\mathcal{A}_\mathcal{V}$ is comparator-adaptive, our entire algorithm is comparator-adaptive. 

Controlling $\lregret_T^\mathcal{Z}$ is an online linear optimization
(OLO) task. For $\mathcal{A}_\mathcal{Z}$, it suffices to use any
OLO algorithm on the unit ball that is
\emph{delay-tolerant}, by which we mean that it satisfies
$
\regret_T(\u) = \mathcal O \big(\sqrt{\Lambda_T} \big).
$
If such an algorithm is used with approximate gradients such that $\|\hat \g_t- \g_t \|\leq \eps$, then it enjoys the bound
$
\regret_T(\u) = \mathcal O \big(\sqrt{\Lambda_T} + T \eps \big) \, .
$
In the remainder of this section, we will present a one-dimensional
algorithm for learning the range, such that, when combined with a
delay-tolerant algorithm for learning the direction, we obtain the
following comparator-adaptive guarantee, which is proved in
Appendix~\ref{app:black-box}:
\begin{theorem}%
	\label{th:informaldelaycompddim}
	Suppose $\mathcal A_{\mathcal Z}$ is a delay-tolerant algorithm,
	and $\mathcal A_{\mathcal V}$ is Algorithm~\ref{alg:pwa},
	defined below and tuned with any $\nu > 0$ and error parameter
	$\eps \geq \| \hat \g_t - \g_t \|$. Then the combination of
	$\mathcal A_{\mathcal Z}$ and $\mathcal A_{\mathcal V}$ by the
	black-box reduction described above (i.e.\
        Algorithm~\ref{alg:black box} in Appendix~\ref{app:black-box}) satisfies
	\begin{equation*}
		\regret_T(\u) 
		\leq \nu + \|\u\| \, \mathcal B(T) \, ,
		\;  \text{where} \; \,  
		\mathcal B(T) = \mathcal O\bigg(  \eps T    +\sqrt{\Lambda_T\ln\Big(1 + \frac{\|\u\|}{\nu}TGD(\Graph)\Big)}\bigg).
	\end{equation*}
\end{theorem}
As shown by Proposition~\ref{prop:adadelaydist} in the Appendix, the
Ada-Delay-Dist algorithm of \citet{hsieh2020multiagent} is
delay-tolerant and can therefore be used for $\mathcal A_{\mathcal Z}$.
It remains to adapt to scale, which requires designing a suitable
one-dimensional algorithm $\mathcal A_{\mathcal V}$. 

Learning the scale is actually a special case of the general problem of
designing a comparator-adaptive algorithm, in which the gradients are
projected down to one dimension via $h_t = \inner{\z_t}{\g_t}$ and 
approximate gradients correspond to $\smash{\hat{h}_t =
	\inner{\z_t}{\hat \g_t}}$. If $\|\hat \g_t - \g_t\| \leq \epsilon$, then
we also have $|\hat{h}_t - h_t| \leq \|\z_t\|\|\hat \g_t - \g_t\| \leq
\epsilon$. It therefore inherits all the difficulties of dealing with
approximate and delayed gradients. 

Let us first sketch the main difficulties. The obvious approach to
handling approximate gradients, which would work well if we did not aim
for comparator-adaptivity, is to observe that 
\begin{equation}\label{eq:standard inexact idea}
	\sumT \inner{\w_t - \u}{\g_t} = \sumT\inner{\w_t - \u}{\hat{\g}_t} + \sumT\inner{\w_t - \u}{\g_t - \hat{\g}_t}.
\end{equation}
Then use a standard algorithm to control $\sumT\inner{\w_t -
	\u}{\hat{\g}_t}$, and attempt to bound $\inner{\w_t - \u}{\g_t -
	\hat{\g}_t}$. While this is possible in
expectation for stochastic encodings by making $\hat \g_t$ an unbiased
approximation of $\g_t$, for deterministic encodings it is not clear how
$\inner{\w_t - \u}{\g_t - \hat{\g}_t}$ can be bounded by a term that
scales with $\|\u\|$. Scaling with $\|\u\|$ is crucial, as we will
exploit the comparator-adaptive property of our algorithms to learn a
$\Qset$-partition. We therefore cannot use this approach.

The second difficulty is due to the fact that gradients may be
unavailable at prediction time due to the delayed feedback. Considerable
work has been done in the delayed feedback setting to tune the learning
rate of standard OCO algorithms to deal with
missing gradients \citep{joulani2016delay, hsieh2020multiagent}.
Unfortunately, these existing approaches do not work for
comparator-adaptive algorithms. For a more in-depth discussion of the difficulties faced in designing a comparator-adaptive algorithm which can handle missing gradients we refer the reader to Appendix \ref{sec:OGD DOCO-JC}. To resolve the aforementioned difficulties
simultaneously, we provide a new one-dimensional comparator-adaptive
algorithm given in Algorithm~\ref{alg:pwa}, which can be used to select
$v_t$. It turns out that these predictions $v_t$ can be computed in linear time
(see \eqref{eq:predictions1dcompute} in Appendix
\ref{app:compgraph}). In the algorithm and the discussion below,
$\Sset_{I_t}(t) \subset \{1,\ldots,t-1\}$ denotes the set of indices of
gradients that are available at node $I_t$ in round~$t$. Similarly,
$\gamma(s) = \{1,\ldots,s-1\} \setminus \Sset_{I_s}(s)$ is the set of
indices of gradients that were missing in round $s$ at node $I_s$, and
$\gamma_{I_t}(s) = \gamma(s) \cap \Sset_{I_t}(t)$ is the set indices of
gradients that were missing at node $I_s$ in round $s$, but are
available at node $I_t$ in round $t$. The regret of
Algorithm~\ref{alg:pwa} is bounded by the following result, whose proof
can be found in Appendix \ref{app:compgraph}.
\begin{theorem}%
	\label{th:informal1d}
	Let $\Lambda_T^h = \sumT \left(h_t^2 + 2|h_t|\sum_{i \in \gamma(t)} |h_i|\right)$. Algorithm \ref{alg:pwa}, tuned with any $\nu > 0 $ and $\eps > 0$ such that $|\hat h_t - h_t| \leq \eps$, satisfies for any $u \geq 0$,
	\begin{equation*}
		\sumT (v_t - u) h_t 
		\leq \nu + u \, \mathcal B^h(T)
		\;  \text{where} \; \,  
		\mathcal B^h(T) = 
		\mathcal O\bigg( \eps T + \sqrt{\Lambda_T^h\ln\Big(1 + \frac{u}{\nu}TGD(\Graph)\Big)}\; \bigg).
	\end{equation*}
\end{theorem}
The fact that the regret of Algorithm \ref{alg:pwa} scales with $u$ is a crucial property that we will use to a $\Qset$-partition, in particular the property that for $u = 0$ the regret is $\nu$ will be repeatedly used. 

We proceed to discuss the main ideas behind Algorithm~\ref{alg:pwa} and
Theorem~\ref{th:informal1d}.
\begin{algorithm}[t]
	\caption{Comparator-Adaptive Algorithm on a Graph for $d=1$}\label{alg:pwa}
	\KwIn{$\nu > 0$,
		upper bound $G$ on $\max_t \|\g_t\|$,
		error parameter $\varepsilon > 0$ \\
		\textbf{Initialize:} $\Sset_n(1) = \varnothing$ and $\gamma_{n}(t) = \varnothing$ for all time-steps $t$ and all nodes $n \in \Nset$, set distribution $\mathrm d\rho(\eta) = \exp(-\eta^2)/Z \mathrm d\eta$ over $\eta \in \left[0, a\right]$, where $a = ((G+\varepsilon)20(1 + 2 D(\Graph)))^{-1}$ and define $Z = \int_0^a \exp(-\eta^2) \mathrm d\eta$.}
	\For{$t = 1 \ldots T$}{
		~~Play $v_t = \E_{\eta \sim \rho}\left[\nu\exp\left(-\sum_{s \in \Sset_{I_t}(t)}\big(\eta (\hat{h}_s + \varepsilon) + \eta^2 (\hat{h}_s + \varepsilon)^2 + 2\eta^2 \hat \zeta_{I_t}(s)\big)\right)\eta\right]$ \\
		For all $n \in \Nset$: send messages, receive messages, update 
		$S_n(t+1)$, and update $\gamma_{n}(s)$  for all $s \in S_n(t+1)$.
	}
\end{algorithm}
One of the essential parts in deriving any
comparator-adaptive algorithm is designing a potential function~$\Phi_T$. To see
how the potential function is used, suppose that we could get a sequence of
predictions $v_1, \ldots, v_T$ that satisfy
\begin{equation}\label{eqn:regretcondition}
\nu -\sumT v_t h_t \geq
\Phi_T\big(-\sumT (\hat{h}_t + \varepsilon)\big)
\qquad \text{for some $\nu > 0$.}
\end{equation}
Then these predictions would satisfy the regret bound
  $\sumT(v_t - u) h_t \leq \nu + \Phi^\star_T\left(u \right) + 2 u
  \varepsilon T$,
where $\Phi^\star_T$ is the convex conjugate of $\Phi_T$.
To see this, recall
Fenchel's inequality $\Phi_T(x) + \Phi^\star_T(u) \geq
xu$, which implies
	$\Phi_T\Big(-\sumT (\hat{h}_t + \varepsilon) \Big)
	\geq  -\Phi^\star_T\left(u \right) - u \sumT (\hat{h}_t + \varepsilon)
	\geq  -\Phi^\star_T\left(u \right) - 2 u\varepsilon T -  u \sumT
        h_t$,
and combine with \eqref{eqn:regretcondition} to obtain the bound on the regret.
We therefore require a potential $\Phi_T$ for which we can satisfy
\eqref{eqn:regretcondition} and for which $\Phi^*_T(u)$ is small enough.
Now, suppose that we could bound the increase in potential per round by
\begin{equation}\label{eqn:decreasingpotential}
	\Phi_t\bigg(-\sumt (\hat{h}_s + \varepsilon) \bigg) \leq
	\Phi_{t-1}\bigg(-\sum_{s=1}^{t-1} (\,\hat{h}_s + \varepsilon) \bigg) -
	v_t h_t.
\end{equation}
Then summing over $t$ would lead to the desired inequality
\eqref{eqn:regretcondition} with $\nu = \Phi_0(0)$. 
But herein lies exactly the technical challenge caused by the missing
gradients. To guarantee
\eqref{eqn:decreasingpotential}, existing comparator-adaptive algorithms
base their prediction for round $t$ on knowledge of $\Phi_{t-1}$, but
the missing gradients prevent us from doing the same. Instead, we have
to use whatever gradients are available at prediction time. 

\begin{sloppypar}
To account for the missing gradients, our predictions include a
correction term $\smash{ \hat \zeta_{I_t}(s) = |\hat{h}_s +
\varepsilon|\sum_{i \in \gamma_{I_t}(s)}|\hat{h}_i + \varepsilon)|}$,
which we incorporate in the predictions $v_t$ defined in
Algorithm~\ref{alg:pwa}.
Similar to the correction for approximate gradients, the correction for
missing gradients decreases the effective learning rate of our algorithm. The corrections play a crucial role in %
our potential function:
\begin{equation}\label{eq:potential}
	\Phi_t\bigg( -\sum_{s = 1}^{t} \big(\hat{h}_s + \varepsilon\big)\bigg)
	= \E_{\eta \sim \rho}\left[\nu\exp\left(-\sum_{s = 1}^{t}\left(\eta (\hat{h}_s + \varepsilon) + \eta^2 (\hat{h}_s + \varepsilon)^2 + 2\eta^2 \hat \zeta(s) \right)\right)\right],
\end{equation}
where $\hat \zeta(s) = |\hat{h}_s + \varepsilon| \sum_{i \in \gamma(s)}|\hat{h}_i + \varepsilon| $. 
The potential function includes a similar correction term as our
predictions, with the difference that the potential corrects for all
missing gradients, not just the ones available at the active node.
Together, these corrections allow us to establish
\eqref{eqn:decreasingpotential}:
\begin{restatable}{relemma}{lempotentialround}
	\label{lem:potentialround}
	Suppose $\|\z_t\| \leq 1$, $\|\g_t\| \leq G$, and $\|\hat{\g}_t -
	\g_t\| \leq \varepsilon$ for all $t$. Then the predictions $v_t$
        defined in Algorithm~\ref{alg:pwa} satisfy 
        \eqref{eqn:decreasingpotential}.
\end{restatable}
\end{sloppypar}
The proof of Lemma \ref{lem:potentialround} (see
Appendix~\ref{app:compgraph}) involves carefully tracking which
gradients are missing. Whereas the analysis of standard
comparator-adaptive algorithms relies on an inequality called the prod
bound \citep[Lemma 2.4]{cesa2006} to obtain an analogue of
\eqref{eqn:decreasingpotential}, the standard prod bound fails in the
presence of missing gradients. The key to our proof is therefore a novel
inequality given in Lemma \ref{lem:usefulineq} in
Appendix~\ref{app:compgraph}, which substantially generalizes the prod
bound.
Finally, it remains to show that $\Phi^\star_T(u)$ is small enough,
which we do in Lemma~\ref{lem:convex conjugate potential} in
Appendix~\ref{app:compgraph}. Together, the above provides a comparator
adaptive algorithm which can handle approximate and missing gradients.%

\section{Limited Communication and Optimality}\label{sec:limited_com}

We proceed to construct both deterministic and stochastic communication
strategies for the gradients, which can be used to apply
Algorithm~\ref{alg:pwa} in the DOCO-JC setting. We will restrict
attention to communication strategies in which nodes send and
receive messages containing approximate gradients. We say an algorithm uses the \emph{standard forwarding strategy} if
every node, upon receiving a gradient that it has not seen yet,
immediately forwards the gradient to all its neighbors. To enable this
strategy, we assume that the messages containing the gradients include
meta-data with a unique identifier, e.g., the time-step at which they
were first sent. We do not account for this meta-data in the discussion
below, because it may already be naturally present in the network
protocol or otherwise it can be encoded at a minor overhead of $\mathcal
O(\log T)$ additional bits.

Under the standard forwarding strategy, a single node sends at most $D(\Graph)$ distinct messages at a time.
Conversely, there exists an activation sequence under which a node will forward $D(\Graph)-1$ messages at the same time. Indeed, on a path of length $D(\Graph)$ in the graph, consider an activation sequence selecting adjacent nodes, going from one end of the path to the other. Using the standard forwarding strategy, the penultimate node forwards the $D(\Graph)-1$ previous messages at the same time.	
Accordingly, under a total $b$-bit constraint on the bandwidth, we assume that the $b$ bits are divided into $D(\Graph)$ slots of $ k = \lfloor b / D(\Graph) \rfloor$ bits, each slot corresponding to a message.

\paragraph{Deterministic encodings}

We first provide upper (Theorem~\ref{th:fulldetencoding}) and
lower (Theorem~\ref{th:lower_boundI}) bounds on the regret for
deterministic encodings. A possible encoding is to fix a cover of the set
of possible gradients, and communicate the element of the cover to which the gradient belongs;
see Appendix~\ref{app:deterministic_enc} for more details. %
The approximate gradients $\hat \g_t$ obtained from this encoding are
then given as inputs to the black-box reduction, with AdaDelay-dist
(from \citet{hsieh2020multiagent}]) as $\mathcal{A}_\mathcal{Z}$ and
Algorithm~\ref{alg:pwa} as $\mathcal{A}_\mathcal{V}$; we tune
Algorithm~\ref{alg:pwa} with $\eps = 3 \cdot 2^{- k /  \lfloor b / D(\Graph)
\rfloor} G$ and the upper bound $4G$ on $\|\hat \g_t \|$. As detailed in
Appendix~\ref{app:deterministic_enc}, these values are valid upper
bounds on the error and the norm of the encodings, and allow us to apply
Theorem~\ref{th:informaldelaycompddim}, and obtain the following
guarantee:

\begin{restatable}[Regret Bound with Deterministic Coding]{retheorem}{thfulldetencoding}\label{th:fulldetencoding}
	Using $k = \lfloor b / D(\Graph) \rfloor$ bits per gradient, the algorithm described above satisfies
	\begin{equation*}
		\regret_T(\u) \leq  \nu + \|\u\| \, \mathcal B(T), 
		\quad \text{where} \quad
		\mathcal B(T) = \tilde {\mathcal O} \bigg(\sqrt{\Lambda_T
		} 
		+ T 2^{-b / (d D(\Graph))} G \bigg) \, .
	\end{equation*}
\end{restatable}

In Section~\ref{app:deterministic_enc}, we also propose a simpler
per-coordinate encoding. This more practical encoding comes at the cost
of an extra $\sqrt{d}$ factor in the second term of the regret bound,
which is acceptable when $b$ is very large.
We further provide the following matching lower bound for
a natural class of algorithms we call gradient-oblivious; see
Appendix~\ref{app:lower-bounds} for a detailed discussion.
\begin{restatable}[Lower Bound I: Deterministic Encoding]{retheorem}{thlowerboundI}\label{th:lower_boundI}
	There exists an activation sequence such that for any
	gradient-oblivious algorithm using a deterministic encoding with
	$\lfloor b / D(\mathcal G) \rfloor$ bits per gradient, with $\smash{\regret_T(  \mathbf 0) \leq \nu}$, and for any comparator norm $U \geq 0$, for $T$ large enough, there exists a comparator $\u \in \R^d$ such that $\| \u \| = U$ and
	\begin{equation*}
		\sup_{G\text{-Lipschitz losses  }}  \!  \regret_T(\u) 
		\geq \max\Big( 0.15 \, \| \u \| G\sqrt{D(\Graph) \, T
                \ln \Big(\frac{\|\u\|^2  T }{72 \nu^2 D(\Graph) }
                \Big)} , \, T 2^{-b / (d D(\Graph))} \, \|\u \|  G \Big) \, . 
	\end{equation*}
\end{restatable}
Since $\Lambda_T \leq 3 G^2 D(\Graph)T$, the upper bound in
Theorem~\ref{th:fulldetencoding} matches this lower bound up to
multiplicative constants and lower order terms. A notable feature of
both the upper and lower bounds is the term containing $T 2^{-b / (d
D(\Graph))}$, which shows that we need roughly $b = \Theta(d
D(\Graph)\log_2(T))$ bits to get non-trivial regret. This means that,
for large dimensions $d$, the number of bits $b$ must also be large. The
proof of Theorem~\ref{th:lower_boundI} is in
Appendix~\ref{app:lower_bound_det}. It consists of two parts: we obtain
the first term in the maximum by modifying a lower bound for
norm-adaptive OCO from \citet{orabona2013dimension} to incorporate the
effect of the graph structure, which adds a $\sqrt{D(\Graph)}$
multiplicative factor compared to the original lower bound. The second
term in the maximum, which is linear in $T$, is new and arises from the
communication limit $k$ on the number of bits that can be transmitted
per gradient.

\paragraph{Stochastic Encodings}\label{sec:stoch_enc}

As discussed above, deterministic encodings require a large number of
bits, which grows at least linearly with $d D(\Graph)$. In the regime
where $b \leq d D(\Graph)$, a better solution is to inject randomness
into the encodings, which bypasses the lower bound from
Theorem~\ref{th:lower_boundI}. It turns out that near-optimal guarantees
can be obtained along with a straightforward analysis and implementation
by combining two known techniques. The first technique may be called
\emph{sparsification} and consists of sampling (uniformly at random) a
single coordinate of the gradient vector to be communicated. Encoding
the index of this coordinate requires $\ceil{\log_2 d}$ bits. The second
technique may be called \emph{$p$-level stochastic quantization}. It
consists of truncating the gradient coordinate to be transmitted to the
first $p$ digits in its binary expansion. We use this with
 $p = \ceil{\log_2(d)}$. Finally, to reduce the variance, we
repeat this construction $m = \Theta (k / \log_2 d)$ times, sending less
than $k$ bits per vector in total. We refer to the joint construction as
\emph{sparsified quantization} with precision $p$ and number of
repetitions $m$. See Appendix~\ref{sec:details stoch encoding} for a
detailed account of the scheme. Very similar constructions have
previously been used by \citet{mayekar2020limits} and
\citet{acharya2019distributed} in related contexts;
Appendix~\ref{sec:details stoch encoding} contains a detailed
comparison. In spite of the simplicity of the construction, we show that
sparsified quantization gives near-optimal theoretical guarantees:
%
\begin{restatable}{retheorem}{thstochasticupper}\label{th:stochasticupper}
        For any vector $\x \in \mathcal B_2(G)$, sparsified quantization
        with precision $p = \lceil \log_2 (d) \rceil$ and $m = \lfloor k
        / (3 \lceil \log_2 (d) \rceil  + 2) \rfloor$ repetitions
        produces a (randomized) approximation $\widehat \x$ that
        satisfies $\| \hat \x \| \leq 2dG$ and $\E\!\big[ \|\widehat \x
        - \x \|^2 \big] \leq (2d/m)  \| \x \|^2 +  G^2 / m = O\big((\log
        d) d / k\big),$ provided that the number of bits per vector is
        at least $k\geq 3 \lceil \log_2 (d) \rceil + 2$.
\end{restatable}

The approximate gradients $\hat \g_t$ obtained from sparsified
quantization (Theorem~\ref{th:stochasticupper}) are then used as inputs
to the black-box reduction, with AdaDelay-dist (from
\citet{hsieh2020multiagent}) as $\mathcal{A}_\mathcal{Z}$ and
Algorithm~\ref{alg:pwa} as $\mathcal{A}_\mathcal{V}$; we tune
Algorithm~\ref{alg:pwa} with $\eps = 0$ and the upper bound $2dG$ on
$\|\hat \g_t \|$. Using this algorithm, we obtain the following result (see Appendix~\ref{app:proof fullstoch} for a full proof).

\begin{restatable}[Regret bound with Stochastic Encoding]{retheorem}{thfullstochencoding}\label{th:fullstochencoding}
	Using $k = \lfloor b / D(\Graph) \rfloor$ bits per gradient, the algorithm described above satisfies
	\begin{equation*}
		\E[\regret_T(\u)] \leq \nu +  \| \u \| \; \mathcal B(T)
		\quad \text{where} \quad
		\mathcal B(T) = \tilde {\mathcal O} \bigg(
		G\sqrt{\Big( 1 + \frac{d D(\mathcal G)}{b}\Big)D(\mathcal G)T
		} \; 
		\bigg) \, . 
	\end{equation*}
\end{restatable}
The following theorem is a matching lower bound, up to log factors, for the natural class of gradient-oblivious algorithms; Appendix~\ref{app:lower_stoch_enc} contains a definition, as well as a proof of the theorem. 

\begin{restatable}[Lower bound II: Stochastic Encodings]{retheorem}{thlowerboundII}\label{th:lower_boundII}
	For any gradient-oblivious algorithm using $\smash{\lfloor b / D(\mathcal G) \rfloor}$ bits per gradient, 
	there exists an activation sequence such that, for any $U > 0$ there exists a sequence of losses and a comparator $\u \in \R^d$ such that $\| \u \| = U $ and
	\begin{equation*}
		\E\!\big[\regret_T(\u) \big] \geq c \, \| \u \|  G \sqrt{ \Big( 1 + \frac{d D(\mathcal G)}{b}\Big)  \,  D(\Graph) T }  \, .
	\end{equation*}
\end{restatable}
To summarize the proof, the lower bound for OCO is $\| \u \|G \sqrt{T}$,
the encodings add a factor of $\sqrt{1 + d D(\Graph) / b}$, and the
delays add another $\sqrt{D(\Graph)}$ factor on top. In
Appendix~\ref{app:lower_stoch_enc}, we analyze in detail how each characteristic of the setting (graph and encoding) affects the hardness.

\section{Learning a \texorpdfstring{$\Qset$}{Qset}-Partition}\label{sec:partitionlearning}

As discussed in the introduction, it can be highly suboptimal for agents
to wait for gradients that take too long to arrive. Instead, the graph
should be partitioned according to a $\Qset$-partition. In this section
we show how to exploit the comparator-adaptive property of our
algorithms to learn a $\Qset$-partition.

For a fixed subgraph $\mathcal F \subseteq \mathcal G$, consider the
DOCO-JC problem restricted to that subgraph, that is, discarding all
gradients and communications coming from nodes outside $\mathcal F$.
Given a general algorithm for the DOCO-JC setting,  we denote by
$\w_t^{\mathcal F}$ the iterate generated by the algorithm restricted to
$\mathcal F$. We define $\Lambda(\mathcal F) = \sum_{t : I_t \in
\mathcal F} \big(\|{\g}_t\|^2 + 2\|\g_t\|\sum_{i \in \gamma(t, \mathcal
F)}\|\g_i\|\big)$, with $\gamma_{\mathcal F}(t) = \big[t\big(\mathcal
F\big) - 1\big] \setminus \Sset_{I_t}(t, \mathcal F)$, and $t(\mathcal
F) = \{s: I_s \in \mathcal F\}$, and where $\Sset_{I_t}(t, \mathcal F)$
is the set of indices of gradients that have been observed by $I_t$
before round $t$. Using the typical bounds we obtain in this article, e.g., in Theorem~\ref{th:fulldetencoding}, we upper bound its regret as $\regret_\F(\u) = \sum_{t: I_t \in
\mathcal F} \left(\ell_t(\w_t) - \ell_t(\u^{\mathcal F})\right) =
\tilde{\O}(\|\u\|\sqrt{\Lambda(\mathcal F)})$. (We neglect the encoding costs here for the sake of simplicity.) %
Then, using this approach, we could fix an oracle partition $\mathcal P
= \{\F_1,\ldots,\F_r\}$ of the graph into disjoint subgraphs and apply
this strategy on each subgraph. This splits the DOCO-JC task into~$r$
independent subtasks, and the total joint regret is simply the sum of
joint regrets of the subtasks:
\begin{equation*}\label{eq:fixed partition}
	R_T(\u) = \sum_{\mathcal F \in \Pset}R_\F(\u)
        =  \sum_{\mathcal F \in \Pset} \tilde{\O}\left(\|\u\|\sqrt{\Lambda(\mathcal F)}\right).
\end{equation*}
(There is even some extra flexibility, which is that each subtask $\F$
could have different comparator parameters~$\u^\mathcal{F}$.)
An apparent drawback of this strategy is that each node gets access to
less information. Whether partitioning is worth it depends on the activation sequence. This raises an issue of adaptation, as the activation sequence is not known in advance.

\paragraph{Iterate Addition} To adapt to the activation sequence we
exploit the following special property of comparator-adaptive
algorithms, observed by \citet{cutkosky2019combining}. For an algorithm
$\mathcal A$, denote by $\w_t^{\Aset}$ its predictions and by
$\lregret_T^\Aset(\u) = \sumT\inner{\w_t^\Aset - \u}{\g_t}$ its
linearised regret. Then consider two algorithms $\mathcal A$ and
$\mathcal B$ that both have constant regret at most $\nu$ against the null comparator: $\lregret_T^\Aset(\0) \leq \nu$ and $\lregret_T^\Bset(\0) \leq \nu$. 
Then simply playing $\w_t = \w_t^\Aset + \w_t^\Bset$ ensures that 
\begin{equation*}
	\regret_T(\u) \leq \sumT\inner{\w_t - \u}{\g_t} =
        \min_{\substack{\x,\y\\
        \x+\y = \u}} \lregret_T^\Aset(\x) + \lregret_T^\Bset(\y)
        \leq \nu + \min_{\Kset \in \{\Aset, \Bset\}} \lregret_T^\Kset(\u) \, , 
\end{equation*} 
where the second inequality comes from minimizing over $(\x = \0,\y =
\u)$ and $(\x = \u, \y = \0)$. We can choose an arbitrary collection of
subgraphs $\Qset$, play $\w_t = \sum_{\mathcal F\in \Qset}
\w_t^{\mathcal F} \mathds{1} \{ I_t \in \mathcal F \}$ and exploit the
property that $\smash{\lregret_{\mathcal F}(\0) \leq \nu}$ to learn how
to partition the graph to minimize the regret. 
The aforementioned predictions combined with ignoring messages older than $D_\Qset$ rounds, using $k = \lfloor b /
D_\Qset \rfloor$ bits per gradient, and our new comparator-adaptive algorithm with deterministic encoding yields the following result, 
proved in Appendix~\ref{app:partition learning}:
\begin{restatable}[Learning a $\Qset$-Partition, deterministic encoding]{retheorem}{thpartidet}\label{th:partidet}
	Let $\Qset $ be a collection of subgraphs of $\mathcal G$. Suppose the learner uses the algorithm of Theorem~\ref{th:fulldetencoding} with deterministic encodings for each subgraph $\mathcal F
	\in \Qset$, discarding any message older than $D_{\Qset} =
	\max_{\mathcal{F} \in \Qset}D(\mathcal{F})$ rounds. Then,
	setting $\nu = 1 / |\Qset|$ and $k = \lfloor b / D_\Qset \rfloor$, and playing $\w_t =
	\sum_{\mathcal F \in \Qset} \w_t^{\mathcal F} \mathds{1} \{ I_t
	\in \mathcal F \}$ guarantees that
	\begin{equation*}
		\sum_{j = 1}^r \regret_{\F_j}(\u_j)
		=   \O\Bigg(\sum_{ j = 1}^r \|\u_j\|\bigg(\sqrt{\Lambda(\mathcal F_j)\ln\big(1 + |\Qset| D(\F_j) \|\u_j\|T_jG\big)} + 2^{-b / (d D_{\Qset} )}T_jG\bigg)\Bigg)
	\end{equation*}
	for any $\Qset$-partition $\{\mathcal F_1, \dots,
	\mathcal F_r \}$ and any $\u_1 \dots, \u_{r} \in \reals^d$.
\end{restatable}
An analogous result holds in expectation for stochastic encodings by using the
algorithm from Theorem~\ref{th:fullstochencoding} for each subgraph.
(See Theorem \ref{th:partistoch} in Appendix \ref{app:partition
learning}.) In Appendix \ref{sec:example collection graphs} we provide
an example collection of subgraphs $\Qset$ with which the learner can adapt to
the activation sequences from Figures~\ref{fig:galaxiesFarFarAway} and
\ref{fig:embeddedGalaxies}. In case the full graph is included in
$\Qset$-partition, i.e.\ $\Graph \in \Qset$, we have $D_{\Qset} = D(\Graph)$. However, notice that the learner may choose not to include $\Graph$ in $\Qset$ to increase the number of bits available to encode each gradient. This in turn allows the learner to improve the regret in some cases, as using more bits per gradient improves the regret bound for the subgraphs. 

\section{Conclusion}\label{sec:conclusion}

We provided a comparator-adaptive algorithm for the DOCO-JC setting. We
provided upper and lower bounds for deterministic and stochastic encoded
gradients and we demonstrated how to exploit the comparator-adaptive
property of our algorithm to learn the best partition in a subset of partitions of the graph.
An interesting direction to improve the communication strategy would be to send information chosen more efficiently, instead of systematically forwarding every gradient.
This might be achieved by implementing other protocols, or error-feedback approaches; see the recent work by \citet{cao2021decentralized, xie2020cser}. Another
limitation is the assumption in the protocol that all agents communicate
in each round, which may be slow because it requires synchronization.
However, as long as the maximum delay before each gradient is observed
is bounded, this issue can be overcome by simply changing the
$D(\Graph)$ in the parameter settings of Algorithm~\ref{alg:pwa} to the
maximum delay. If no such bound is known beforehand, the problem becomes more complex, but a related problem for comparator-adaptive algorithms is learning $G$, rather than providing $G$ up front. This problem has been studied in \citep{cutkosky2019artificial, mhammedi2020lipschitz} and perhaps their techniques transfer to learning the maximum delay.

\bibliography{adol.bib}

\appendix

\input{adol_app.tex}

\end{document}

%% file: commands.tex
\usepackage{parskip}

\usepackage{bm}

\usepackage{enumerate} 
\usepackage{url} 
\usepackage{float} 
\usepackage{color} 

\usepackage{listings}
\usepackage{algorithm}
\usepackage{comment}
\usepackage{dsfont}
\usepackage{upgreek}

\newcommand{\TODO}[1]{%
	\ifmmode
	\text{\textcolor{red}{TODO: #1}}
	\else
	\textcolor{red}{TODO: #1}
	\fi
}

\newcommand{\Dirk}[1]{%
	\ifmmode
	\text{\textcolor{green}{Dirk: #1}}
	\else
	\textcolor{green}{Dirk: #1}
	\fi
}
\renewcommand{\P}{\mathbb P}
\newcommand{\R}{\mathbb R}
\newcommand{\N}{\mathbb N}

\newtheorem{theorem}{theorem}
\newtheorem{lemma}[theorem]{Lemma}
\newtheorem{proposition}[theorem]{Proposition}
\newtheorem{corollary}[theorem]{Corollary}

\DeclareMathOperator{\E}{\mathbb E}

\DeclareBoldMathCommand{\w}{w}
\DeclareBoldMathCommand{\u}{u}
\DeclareBoldMathCommand{\b}{b}
\DeclareBoldMathCommand{\g}{g}
\DeclareBoldMathCommand{\x}{x}
\DeclareBoldMathCommand{\y}{y}
\DeclareBoldMathCommand{\s}{s}
\DeclareBoldMathCommand{\I}{I}
\DeclareBoldMathCommand{\0}{0}
\DeclareBoldMathCommand{\z}{z}
\DeclareBoldMathCommand{\G}{G}

\renewcommand{\O}{\mathcal{O}}

\newcommand{\F}{\mathcal{F}}
\newcommand{\Bset}{\mathcal{B}}

\newcommand{\Kset}{\mathcal{K}}
\newcommand{\Pset}{\mathcal{P}}

\newcommand{\Qset}{\mathcal{Q}}

\newcommand{\Nset}{\mathcal{N}}
\newcommand{\Sset}{\mathcal{S}}
\newcommand{\Aset}{\mathcal{A}}

\newcommand{\reals}{\mathbb{R}}
\newcommand{\regret}{\mathcal{R}}
\newcommand{\lregret}{\tilde{\mathcal{R}}}
\newcommand{\domainw}{\mathcal{W}}

\newcommand{\sumT}{\sum_{t=1}^T}
\newcommand{\sumt}{\sum_{s=1}^t}

\newcommand\inner[2]{\langle #1, #2 \rangle}
\newcommand{\half}{\tfrac{1}{2}}

\newcommand{\etahat}{\hat{\eta}}
\newcommand{\Graph}{\mathcal{G}}
\newcommand{\ceil}[1]{\left\lceil #1 \right\rceil}
\newcommand{\floor}[1]{\left\lfloor #1 \right\rfloor}

\newcommand{\eps}{\varepsilon}
\newcommand{\erf}{\textnormal{erf}}

\renewcommand{\leq}{\leqslant}
\renewcommand{\geq}{\geqslant}

\renewcommand{\tilde}{\widetilde}
\renewcommand{\epsilon}{\varepsilon}
\renewcommand{\eps}{\varepsilon}

\renewcommand{\hat}{\widehat}

%% file: adol_app.tex
\section{Analysis of Unprojected Online Gradient Descent in One Dimension in the DOCO-JC Setting}\label{sec:OGD DOCO-JC}

The role of this section is to provide context for the one-dimensional
algorithm in Section~\ref{sec:compgraph}. To do so, we analyse an
unprojected version of Online Gradient Descent (OGD) in a simplified
setting: in one dimension with a constant learning rate $\eta$ on each
node, and with no communication limits. We denote the gradient (a real
number) at time $t$ by $h_t$, and use the notation defined in the introduction.

As mentioned in Section~\ref{sec:compgraph}, one of the principal technical challenges that Algorithm~\ref{alg:pwa} overcomes is that at prediction time, some gradients are not available to the agents. Let us see how missing gradients influence the regret of OGD.
Denote by $H_t(I_t) = \sum_{s \in \Sset_{I_t}} h_s$ the sum of the
gradients available at prediction time on the active node $I_t$ in round
$t$, and denote by $H_t = \sum_{s < t} h_s$ the sum of \emph{all}
gradients before round $t$. 

For comparison, define $w_t$ to be the sequence of predictions that OGD
would output if all gradients were immediately available, that is, $w_t
= -\eta H_t$. These predictions enjoy the standard regret bound
\begin{equation}\label{eq:OGDregret}
	\sum_{t = 1}^T (w_t - u) h_t \leq \frac{|u|^2}{2\eta} + \frac{\eta}{2}\sum_{t=1}^T|h_t|^2
\end{equation}
against any comparator $u\in \R$. To bound the regret of OGD with only the available gradients, $w_t(I_t) = -\eta H_t(I_t)$, observe that
\begin{align*}
	\sum_{t = 1}^T \big(w_t(I_t) - u\big) h_t & = \sum_{t = 1}^T (w_t - u) h_t + \sum_{t = 1}^T (w_t(I_t) - w_t) h_t \\
	& \leq \frac{|u|^2}{2\eta} + \frac{\eta}{2}\sum_{t=1}^T|h_t|^2 + \sum_{t = 1}^T \big(w_t(I_t) - w_t\big) h_t,
\end{align*}
where we used \eqref{eq:OGDregret}.
Recall that $\gamma(t)$ is the set of indices of the missing gradients
at the active agent in round $t$. We have that $(w_t(I_t) - w_t) g_t =
\eta h_t\sum_{s \in \gamma(t)} h_s \leq \eta |h_t|\sum_{s \in \gamma(t)}
|h_s|$, which is roughly the term $\widehat{\zeta}(s)$ that appears in
the definition of our potential function in \eqref{eq:potential}. The regret can then be bounded by 
\begin{equation*}\label{eq:OGDregret2}
	\sum_{t = 1}^T (w_t(I_t) - u) h_t \leq \frac{|u|^2}{2\eta} + \eta \sum_{t=1}^T\bigg(\frac{1}{2} h_t^2 + |h_t|\sum_{s \in \gamma(t)} |h_s|\bigg).
\end{equation*}
An ideal tuning of the learning rate would be to optimize the expression above over $\eta$ and set
\begin{equation*}
	\eta
	 = |u| \bigg(\sumT \Big( h_t^2 + 2|h_t|\sum_{s \in \gamma(t)} |h_s|\Big) \bigg)^{-2} \, ,
\end{equation*}
to obtain
\begin{align*}
	\sum_{t = 1}^T (w_t(I_t) - u) h_t \leq |u|\sqrt{2\sumT\bigg(|h_t|^2 + 2|h_t|\sum_{s \in \gamma(t)} |h_s|\bigg)}.
\end{align*}
This is exactly the type of regret bound that would be suitable to learn a $\Qset$-partition. 
However, this ideal tuning is not possible for two reasons: we know
neither $|u|$ nor $\sumT\left(|h_t|^2 + 2|h_t|\sum_{s \in \gamma(t)} |h_s|\right)$. 

Let us define $\lambda_t = |h_t|^2 + 2|h_t|\sum_{s \in \gamma(t)} |h_s|$
in the following discussion to reduce clutter. Note that the active node
$I_t$ at round $t$ may not be able to compute $\lambda_s$ for some $s <
t$. Indeed, gradients missing at past rounds might not have reached
$I_t$ yet, making it impossible to compute some gradient in $\gamma(s)$. %
Even with knowledge of $|u|$, this rules out natural ideas using learning rate schemes such as $\eta_t = \sqrt{{|u|^2}/{\sum_{s < t} \lambda_s}}$.

As discussed in the introduction, a considerable amount of effort has
been made in the literature to obtain approximations of the optimal learning rate in the delayed OCO setting. However, to our knowledge, none of these existing approaches apply to comparator-adaptive algorithms. Similarly, although adapting to $|u|$ is relatively well understood in the standard setting, there are no straightforward extensions to deal with missing gradients. 

Algorithm~\ref{alg:pwa} solves these challenges simultaneously. As illustrated above, a crucial part of the difficulty lies in trying to adapt the unknown quantity $|h_t|\sum_{s \in \gamma(t)} |h_s|$, which appears in the ideal learning rate.
We accomplish this by showing that only knowing an approximation of this quantity is sufficient: our prediction $v_t$ in Algorithm~\ref{alg:pwa} contains the approximation $\widehat{\zeta}_{I_t}$, which plays a similar role as $|h_t|\sum_{s \in \gamma(t)} |h_s|$, in the sense that it reduces the learning rate of the algorithm to account for the extra uncertainty due to the missing gradients. Our analysis reveals that this approximation does not hurt the regret bound of the algorithm significantly, and that we can still recover the optimal comparator-adaptive regret bound.

\section{Details of Section \ref{sec:compgraph}}\label{app:compgraph}

The predictions $v_t$ in Algorithm~\ref{alg:pwa} can be computed as follows. Let $L_{I_t} = \sum_{s \in \Sset_{I_t}(t)}(\hat{h}_s + \varepsilon)$ and $V_{I_t} =1 + \sum_{s \in \Sset_{I_t}(t)}\big((\hat{h}_s + \varepsilon)^2 + 2\eta^2 \hat \zeta_{I_t}(s)\big)$. Then
\begin{align}\label{eq:predictions1dcompute}
	v_t = \sqrt{\pi}\exp\left(\frac{L_{I_t}^2}{4V_{I_t}}\right)\left(\erf\left(\frac{2aV_{I_t} + L_{I_t}}{2\sqrt{V_{I_t}}}\right) - \erf\left(\frac{L_{I_t}}{2\sqrt{V_{I_t}}}\right)\right)\left(2Z\sqrt{V_{I_t}}\right)^{-1}.
\end{align}
We refer the reader to Appendix B of \citet{koolen2015second} for numerically stable evaluation.

\lempotentialround*
\begin{proof}
	As a first step, observe that by the Cauchy-Schwarz inequality and the condition on $\z_t$ we have that
	\begin{equation*}
	|\hat{h}_t - h_t| = |\inner{\z_t}{\hat{g}_t} - \inner{\z_t}{\g_t}| \leq \|\z_t\|\|\hat{\g}_t - \g_t\| \leq \varepsilon.
	\end{equation*}
	Next, we replace $w_t$ with its definition,
	\begin{align*}
		& \underset{\eta \sim \rho}{\E}\left[\nu\exp\left(-\sum_{s = 1}^{t-1}\left(\eta (\hat{h}_s + \varepsilon) + \eta^2 (\hat{h}_s + \varepsilon)^2 + 2\eta^2 \sum_{i \in \gamma(s)}|(\hat{h}_s + \varepsilon)(\hat{h_i} + \varepsilon)| \right)\right)\right] -  v_{t} h_{t} \\
		& = \underset{\eta \sim \rho}{\E}\left[\nu\exp\left(-\sum_{s = 1}^{t-1}\left(\eta (\hat{h}_s + \varepsilon) + \eta^2 (\hat{h}_s + \varepsilon)^2 + 2\eta^2 \sum_{i \in \gamma(s)}|(\hat{h}_s + \varepsilon)(\hat{h_i} + \varepsilon)| \right)\right)\right] - \\
		& ~~  \underset{\eta \sim \rho}{\E}\left[\nu \exp\left(-\sum_{s \in \Sset_{I_t}(t)}\left(\eta (\hat{h}_s + \varepsilon) + \eta^2 (\hat{h}_s + \varepsilon)^2 + 2\eta^2 \sum_{i \in \gamma_{I_t}(s)}|(\hat{h}_s + \varepsilon)(\hat{h_i} + \varepsilon)| \right)\right)\eta\right] h_{t} \, .
	\end{align*}
	Denote by $\hat \Omega_T$ the sum, which is the central part of the update $w_t$,
	\begin{equation*}
		\widehat \Omega_t = \sum_{s \in \Sset_{I_t}(t)}\Big(\eta (\hat{h}_s + \varepsilon) + \eta^2 (\hat{h}_s + \varepsilon)^2 + 2\eta^2 \sum_{i \in \gamma_{I_t}(s)}|(\hat{h_s} + \varepsilon)(\hat{h_i} + \varepsilon)| \Big) \, .
	\end{equation*}
	We factor this $\exp(- \widehat \Omega_t)$ term in the expression above, so that it is equal to
	\begin{align*}		
		&  \E_{\eta \sim \rho}\Bigg[\nu\exp\big( \! - \!\widehat \Omega_t \big)
		 \! \\
		& ~~ \times \Bigg( \exp\bigg( \! \!- \! \! \sum_{s \in [t-1] \setminus \Sset_{I_t}(t)}\! \!\Big(\eta (\hat{h}_s + \varepsilon) + \eta^2 (\hat{h}_s + \varepsilon)^2 + 2\eta^2 \sum_{i \in \gamma(s)}|(\hat{h}_s + \varepsilon)(\hat{h}_i + \varepsilon)| \Big)\bigg) \\
		& ~~ \times \exp\Big(-\sum_{s \in \Sset_{I_t}(t)}2\eta^2 \sum_{i \in \gamma(s)\backslash\gamma_{I_t}(s)}|(\hat{h}_s + \varepsilon)(\hat{h}_i + \varepsilon)|   \Big)
		- \eta h_t\Bigg)\Bigg] \, . 
	\end{align*}
	Consider the double sum in this last term, right above, 
	\begin{equation*}
		\sum_{s \in \Sset_{I_t}(t)} \sum_{i \in \gamma(s)\setminus\gamma_{I_t}(s)}|(\hat{h}_s + \varepsilon)(\hat{h}_i + \varepsilon)| \, . 
	\end{equation*}
	Let us switch the order of summation here. To do so, note that if a gradient $\hat{h}_i $ is available to node $I_t$ at time $t$, and if it was unavailable when $\hat{g}_s$ appeared, then it would be used in $\gamma_{I_t}(s)$.
	In other words, if $i \in \gamma(s) \setminus \gamma_{I_t}(s)$ then, $i \notin \Sset_{I_t}(t)$.
	Therefore, when $s$ varies in $S_{I_t}(t)$, the set of values taken by $i \in \gamma(s) \setminus \gamma_{I_t}(s) $  is in fact $[t-1] \setminus S_{I_t}(t)$. When switching the sums, we may thus restrict the values taken by $i$ to $[t-1] \setminus S_{I_t}(t)$ and write
	\begin{equation}
		\sum_{s \in \Sset_{I_t}(t)} \sum_{i \in \gamma(s)\setminus\gamma_{I_t}(s)}|(\hat{h}_s + \varepsilon)(\hat{h}_i + \varepsilon)| 
		= \sum_{i \in [t-1] \setminus S_{I_t} }  \sum_{\substack{s \in \Sset_{I_t}(t) \text{ s.t.}  \\    i \in \gamma(s)\setminus\gamma_{I_t}(s) }  } |(\hat{h}_s + \varepsilon)(\hat{h}_i + \varepsilon)| \, . 
	\end{equation}
	Finally, we also switch the notation for $i$ and $s$ in the indices, factor out the common $\hat{h}_s + \varepsilon$ term, and incorporate this sum with the preceding term, obtaining
	\begin{equation*}
		\sum_{s \in [t-1] \backslash \Sset_{I_t}(t)}\bigg(\eta (\hat{h}_s + \varepsilon) + \eta^2 (\hat{h}_s + \varepsilon)^2 + 2\eta^2 |\hat{h}_s + \varepsilon|
		\underbrace{\Big( \sum_{i \in \gamma(s)}|(\hat{h}_i + \varepsilon)|
			+ \sum_{\substack{i \in \Sset_{I_t}(t) \text{ s.t.}  \\    s \in \gamma(i)\setminus\gamma_{I_t}(i) }  } |\hat{h}_i + \varepsilon| \Big)}_{ := \Theta_{I_t}(s)} \, .
		\bigg)
	\end{equation*}
	Furthermore, since a new gradient takes less than $D(\Graph)$ rounds to reach all nodes, there can be only at most $D(\Graph)$ new gradients for which $\gamma_{I_t}(s) \neq \gamma(s) $. Thus the last sum has at most $D(\Graph)$ terms, and the freshly defined $\Theta_{I_t}(s)$ involves at most $2 D(\Graph)$ terms.
	
	We may now apply Lemma~\ref{lem:usefulineq}, with $x = \eta g_t$, to see that for any $\eta$, 
	\begin{align*}
		&\exp\bigg( - \sum_{s \in [t-1] \setminus \Sset_{I_t}(t)}\big(\eta (\hat{h}_s + \varepsilon) + \eta^2 (\hat{h}_s + \varepsilon)^2 + 2\eta^2 |\hat{h}_s + \varepsilon|
		\Theta_{I_t}(s) \Big)
		\bigg) - \eta h_t \\
		&\;  \geq 
		\exp\bigg(-\sum_{s \in [t-1] \backslash \Sset_{I_t}(j)}\left(\eta (\hat{h}_s + \varepsilon) + \eta^2 (\hat{h}_s + \varepsilon)^2 + 2\eta^2 |\hat{h}_s + \varepsilon| \Theta_{I_t}(s) \right)\bigg) \\
		& \;  \times \exp\bigg(-\Big(\eta h_t + \eta^2 h_t^2 + 2\eta^2 \sum_{i \in [t-1] \backslash \Sset_{I_t}(t)}|h_t(\hat{h_i} + \varepsilon)| \Big)\bigg) \, . 
	\end{align*}
	Thus, upon multiplying by $\exp( - \widehat \Omega_t)$, and after integrating over $\eta \sim \rho$, we obtain
	\begin{align*}
		& \E_{\eta \sim \rho}\left[\nu\exp\left(-\sum_{s = 1}^{t-1}\left(\eta (\hat{h}_s + \varepsilon) + \eta^2 (\hat{h}_s + \varepsilon)^2 + 2\eta^2 \sum_{i \in \gamma(s)}|(\hat{h}_s + \varepsilon)(\hat{h}_i + \varepsilon)| \right)\right)\right] -  v_{t} h_{t} \\
		& \geq \E_{\eta \sim \rho}\bigg[\nu\exp\bigg(-\sum_{s \in \Sset_{I_t}(t)}\left(\eta (\hat{h}_s + \varepsilon) + \eta^2 (\hat{h}_s + \varepsilon)^2\right) + 2\eta^2 \sum_{i \in \gamma_{I_t}(s)}|(\hat{h}_s + \varepsilon)(\hat{h}_i + \varepsilon)| \bigg)\\
		&~~ \times \exp\bigg(-\sum_{s \in [t-1] \backslash \Sset_{I_t}(j)}\left(\eta (\hat{h}_s + \varepsilon) + \eta^2 (\hat{h}_s + \varepsilon)^2 + 2\eta^2 |\hat{h}_s + \varepsilon| \, \Theta_{I_t}(s) \right)\bigg) \\
		&~~ \times \exp\bigg(-\Big(\eta h_t + \eta^2 h_t^2 + 2\eta^2 \sum_{i \in [t-1] \backslash \Sset_{I_t}(t)}|h_t(\hat{h}_i + \varepsilon)| \, \Big)\bigg)\bigg].
	\end{align*}
	Define $f(x) = -x - x^2 - 2|x| y$, where $y = \sum_{i \in \gamma(t)} |\eta (\hat{h}_i+\varepsilon)| > 0$. We have that $f'(x) \leq 0$ for $x \geq -\frac{1}{4}$ and $y \in[0, \frac{1}{4}]$. Since $\sum_{i \in \gamma(t)} |\eta (\hat{h}_i+\varepsilon)| \leq 1/4$ by the restriction on $\eta$, the function $f(x)$ is non-increasing for $x > -1/4$. Since $-\frac{1}{4} \leq \eta h_t \leq \eta(\hat{h}_t + \varepsilon)$ we have that $f(\eta h_t) \geq f\big(\eta (\hat{h}_t + \varepsilon)\big)$, which gives us
	\begin{equation*}
		\begin{split}
			& \E_{\eta \sim \rho}\left[\nu\exp\left(-\sum_{s = 1}^{t-1}\left(\eta (\hat{h}_s + \varepsilon) + \eta^2 (\hat{h}_s + \varepsilon)^2 + 2\eta^2 \sum_{i \in \gamma(s)}|(\hat{h}_s + \varepsilon)(\hat{h}_i + \varepsilon)| \right)\right)\right] -  v_{t} h_{t} \\
			& \geq \E_{\eta \sim \rho}\bigg[\nu\exp\left(-\sum_{s \in \Sset_{I_t}(t)}\left(\eta (\hat{h}_s + \varepsilon) + \eta^2 (\hat{h}_s + \varepsilon)^2\right) + 2\eta^2 \sum_{i \in \gamma_{I_t}(s)}|(\hat{h}_s + \varepsilon)(\hat{h}_i + \varepsilon)| \right)\\
			&~~ \times \exp\left(-\sum_{s \in [t-1] \backslash \Sset_{i_t}(j)}\left(\eta (\hat{h}_s + \varepsilon) + \eta^2 (\hat{h}_s + \varepsilon)^2 + 2\eta^2 |\hat{h}_s + \varepsilon| \Theta_{I_t}(s) \right)\right) \\
			&~~ \times \exp\left(-\left(\eta (\hat{h}_t + \varepsilon) + \eta^2 (\hat{h}_t + \varepsilon)^2 + 2\eta^2 \sum_{i \in \gamma(t)}|(\hat{h}_t + \varepsilon)(\hat{h}_i + \varepsilon)| \right)\right)\bigg] \\
			& = \E_{\eta \sim \rho}\left[\nu\exp\left(-\sum_{s = 1}^{t}\left(\eta (\hat{h}_s + \varepsilon) + \eta^2 (\hat{h}_s + \varepsilon)^2 + 2\eta^2 \sum_{i \in \gamma(s)}|(\hat{h}_s + \varepsilon)(\hat{h}_i + \varepsilon)| \right)\right)\right],
		\end{split}
	\end{equation*}
	which completes the proof.
\end{proof}

\begin{theorem}\label{th:pwa}
	Suppose that $ \|\g_t\| \leq G$ and $ \| \hat \g_t - \g_t \| \leq \eps $ for all $t$.  For all $u \in \reals_+$ and $\nu > 0$, Algorithm~\ref{alg:pwa} satisfies the following regret bound:
	\begin{align*}
		& \sumT v_t h_t -\sumT u h_t \leq \nu + 2uT\varepsilon + \\
		&  |u| \max \Bigg\{\,264\,G \tau 
		\ln_+ \! \Big(\frac{312\, |u|G \, \tau}{\nu} \Big) \, , \sqrt{8 \Big(\Lambda^h_t + 24TG\eps + 1\Big)
			\ln_+\! \Big(   \frac{2036 \,  u^2 T \tau G^2}{\nu^2}\Big)} \Bigg\}.
	\end{align*}
	where $\ln_+(x) = \ln(\max(e, x))$ and 
	\begin{equation*}
		\Lambda^h_T = \sum_{t=1}^T \Big( h_t^2 + 2 |h_t|\sum_{s \in \gamma(t)} |h_s| \Big) \, . 
	\end{equation*}
\end{theorem}
\begin{proof}
	First, by repeatedly applying Lemma \ref{lem:potentialround} we find that
	\begin{align}\label{eq:firststeppwa}
		\nu-\sumT v_t h_t \geq \E_{\eta \sim \rho}\left[\nu\exp\left(-\sumT\left(\eta (\hat{h}_t + \varepsilon) + \eta^2 (\hat{h}_t + \varepsilon)^2 + 2\eta^2 \zeta(t) \right)\right)\right].
	\end{align}
	Now, consider the case in which $-\sumT (\hat{h}_t + \varepsilon) \leq \sqrt{2\sumT\left((\hat{h}_t + \varepsilon)^2 + 2\zeta(t)\right)}$. In this case we have:
	\begin{align*}
		\sumT v_t h_t -\sumT u h_t \leq & -\Phi\left(-\sumT (\hat{h}_t + \varepsilon) \right) + \nu -\sumT u h_t\\
		\leq & \nu -\sumT u h_t \\
		\leq & \nu + u T \varepsilon -\sumT u \hat{h}_t \tag{$u \geq 0$}\\
		= & \nu + 2 u T \varepsilon -\sumT u (\hat{h}_t + \varepsilon) \\
		\leq & \nu + 2 u T \varepsilon + u\sqrt{2\sumT\left((\hat{h}_t + \varepsilon)^2 + 2\zeta(t)\right)},
	\end{align*}
	which implies the result.
	
	Next, we consider the case in which $-\sumT (\hat{h}_t + \varepsilon) > \sqrt{2\sumT\left((\hat{h}_t + \varepsilon)^2 + 2\zeta(t)\right)}$. By Fenchel's inequality we have
	\begin{align*}
		-\sumT v_t h_t \geq & \Phi_T\left(-\sumT (\hat{h}_t + \varepsilon)\right) - \nu \\
		\geq & - u \sumT (\hat{h}_t + \varepsilon) - \Phi_T^\star(u) - \nu \\
		\geq & - u \sumT h_t - 2uT\varepsilon - \Phi_T^\star(u) - \nu \, ,
	\end{align*}
	where $\Phi_T^\star$ is the convex conjugate of $\Phi_T$. Using the upper bound on $\Phi^\star(u)$ from Lemma \ref{lem:convex conjugate potential} we get,
	\begin{align*}
		& \sumT v_t h_t -\sumT u h_t \leq \nu + 2uT\varepsilon + \\
		&  \max \Bigg\{u\,44(G+\varepsilon)(2D(\Graph) + 1)\left(\ln\big(|u|44(G+\varepsilon)(D(\Graph) + 1)\big) - 1 + \ln \frac{1}{2}\left(\frac{\pi}{\nu^2}\right)\right) ,\\
		& \sqrt{8u^2 \left(\sumT \Big( \big( \hat{h}_t + \varepsilon\big)^2 + 2\hat \zeta(t)\Big) + 1\right)
		\ln\left(24u^2 \left(\sumT\Big(\big(\hat{h}_t + \varepsilon\big)^2 + 2\hat \zeta(t)\Big) + 1\right) \frac{\pi}{\nu^2} + 1\right)} \Bigg\},
	\end{align*}
	
	\paragraph{Simplifying the bound}
	Since $| \hat h_t  - h_t |\leq \eps$, we have, using $\eps \leq G$ and $1 \leq \tau$, 
	\begin{equation}\label{eq:bounding lag}
		\begin{split}
			(\hat h_t + \varepsilon)^2 + 2\hat \zeta(t) \leq & (h_t + 2\varepsilon)^2 + 2\sum_{s \in \gamma(t)}|(h_t + 2\varepsilon)(h_s + 2\varepsilon)| \\
			\leq & h_t^2 + 4G\eps + 4\varepsilon^2 + 8\tau(G\varepsilon + \varepsilon^2) + 2 |h_t|\sum_{s \in \gamma(t)} |h_s| \\
			\leq &h_t^2 + 2 | h_t | \sum_{s \in \gamma(t)} |h_s| + 24 G \tau \eps  \, . 
		\end{split} \,.
	\end{equation}
	Therefore, by summing over $t \in [T]$:
	\begin{equation*}
		\sumT \Big( \big( \hat{h}_t + \varepsilon\big)^2 + 2\hat \zeta(t)\Big)
		\leq \Lambda^h_t + 24\eps G T 
		\leq 27\, G^2 \tau T \, . 
	\end{equation*}
	We shall use the first inequality to bound the main term inside the square root, and the second cruder bound to bound the term inside the logarithm.
	
	Crudely bounding the other terms with $\eps \leq G$ and $1 \leq \tau$ we get that:
	\begin{align*}
		& \sumT v_t h_t -\sumT u h_t \leq \nu + 2uT\varepsilon + \\
		&  |u| \max \Bigg\{\,264\,G \tau 
		\ln_+ \! \Big(\frac{312\, |u|G \, \tau}{\nu} \Big) \, , \sqrt{8 \Big(\Lambda^h_t + 24TG\eps + 1\Big)
			\ln_+\! \Big(   \frac{2036 \,  u^2 \tau TG^2}{\nu^2}\Big)} \Bigg\}. 
	\end{align*}
\end{proof}

\begin{lemma}\label{lem:usefulineq}
	For $x, y_1, \dots, y_\uptau \in  [-1 /  20(1 + \tau) , 1 / 20(1 + \tau)]$ and $a_i \in [0, 1/20]$, we have
	\begin{equation*}
		\exp\left(\sum_{i = 1}^\tau\left(-y_i-y_i^2-2 a_i |y_i| - 2|xy_i|\right) - x - x^2 \right) \leq \exp \left(\sum_{i = 1}^\tau\left(-y_i-y_i^2-2 a_i |y_i|\right)\right) - x \, .
	\end{equation*}
\end{lemma}
\begin{proof}
	In the case where $x = 0$ the inequality holds trivially. Define the function $f(x) = x + x^2 + \sum_{i = 1}^\tau y_i+y_i^2 + 2|x||y_i|$. We have that
	\begin{align*}
		&\exp\bigg(\sum_{i = 1}^\tau\left(-y_i-y_i^2-2 a_i |y_i| - 2|xy_i|\right) - x - x^2 \bigg) \\
		& = \exp\left(\sum_{i = 1}^\tau\left( -2 a_i |y_i|\right)-f(x)\right)  \\
		& \leq \exp\bigg(\Big(\sum_{i = 1}^\tau -2 a_i |y_i|\big)-f(0)\bigg) - \exp\bigg(\Big(\sum_{i = 1}^\tau -2 a_i |y_i|\Big)-f(0)\bigg)(\partial f(0)) x \\
		& = \exp\bigg(\Big(\sum_{i = 1}^\tau -2 a_i |y_i|\Big)-f(0)\bigg) - \exp\left(\sum_{i = 1}^\tau\left( -y_i - y_i^2-2 a_i |y_i|\right)\right)\left( \partial f(0) \right)x,
	\end{align*}
	where $\partial f(x)$ denotes a subdifferential of $f$ evaluated at $x$ and we used that $\exp(-f(x)) \leq \exp(-f(0)) - \exp(-f(0))(\partial f(0))(x - 0)$ since $\exp (- f )$ is concave by Lemma~\ref{lem:expconcave}.
	If $x>0$ then we set $\partial f(0) = \sum_{i = 1}^\tau 2|y_i| + 1$ and
	\begin{align*}
		& - \exp\left(\sum_{i = 1}^\tau \left(-y_i - y_i^2-2 a_i |y_i|\right)\right)\left(  \sum_{i = 1}^\tau 2|y_i| + 1 \right)x \\
		& \leq -\exp\bigg(\sum_{i = 1}^\tau\left( 2|y_i| - y_i - y_i^2 -2 a_i |y_i|\right) - 4\Big(\sum_{i = 1}^\tau|y_i|\Big)^2\bigg) x \tag{prod bound} \\
		& \leq -\exp\bigg(\sum_{i = 1}^\tau \left(|y_i| - y_i^2 -2 a_i |y_i|\right) - 4\Big(\sum_{i = 1}^\tau|y_i|\Big)^2\bigg)x \\
		& \leq -\exp\bigg(\sum_{i = 1}^\tau \left(|y_i| - 5 y_i^2 -2 \Big(a_i + \frac{4}{20}\Big) |y_i|\right)\bigg)x \tag{$|y_i| \leq \frac{1}{20(\tau +1)}$}\\
		& \leq -\exp\bigg(\sum_{i = 1}^\tau\Big( \frac{1}{2}|y_i| - 5 y_i^2 \Big)\bigg)x \tag{$|a_i|\leq \frac{1}{20}$} \\
		& \leq -x
	\end{align*}
	where the prod bound is $1 + v \geq \exp(v - v^2)$ for $v > -1/2$ (see Lemma 2.4 by \citet{cesa2006}) and the last inequality follows because $\half |v| - 5 v^2$ is %
	non-negative for when $|v| \leq 1 / 10$.
	If $x<0$ then we set $\partial f(0) = 1 - \sum_{i = 1}^\tau 2|y_i|$ and
	\begin{align*}
		& - \exp\left(\sum_{i = 1}^\tau\left( -y_i - y_i^2-2 a |y_i|\right)\right)\left(1 -  \sum_{i = 1}^\tau 2|y_i| \right)x \\
		& \leq - \exp\left(\sum_{i = 1}^\tau \left(-y_i - y_i^2\right) \right)\left(1 -  \sum_{i = 1}^\tau 2|y_i| \right)x \\
		& \leq -\left(\prod_{i = 1}^{\tau}(1 - y_i)\right)\left(1 -  \sum_{i = 1}^\tau 2|y_i| \right)x \tag{repeated prod bound} \\
		& \leq -\left(\prod_{i = 1}^{\tau}(1 + |y_i|)\right)\left(1 -  \sum_{i = 1}^\tau 2|y_i| \right)x \\
		& \leq -\left(\prod_{i = 1}^{\tau}(1 + |y_i|)\right)\left(\prod_{i = 1}^{\tau}(1 - 2|y_i|)\right)x \tag{Weierstrass inequality} \\
		& = -\left(\prod_{i = 1}^{\tau}(1 + |y_i|)(1 - 2|y_i|)\right)x\\
		& \leq -x \tag{$(1 + |y|)(1 - 2|y|) \leq 1$ and $x < 0$},
	\end{align*}
	which completes the proof.
\end{proof}

\begin{lemma}\label{lem:expconcave}
	Define $f : x \mapsto    x + x^2 + \sum_{i = 1}^\tau \left(y_i+y_i^2 + 2 | x | |y_i|\right)$ , the function $g  =\exp(-f)$ is concave on the interval $[ - 1 /10 , \,  1 / 10 ]$, as soon as  $y_1, \dots , y_\tau$ satisfy $|y_i| \leq 1 / (10\tau)$.
\end{lemma}
\begin{proof}
	The function $g : x \mapsto \exp(-f(x))$ is continuous on $\R$, and twice differentiable on $(-\infty, 0)$ and on $(0, + \infty)$ with,
	\begin{equation*}
		g'(x) =
		\begin{cases}		
			&\Big( - 1 -2x + 2\sum_{i = 1}^\tau | y_i | \Big)\exp(-f(x)) \quad \text{if} \quad x < 0 \\
			& \Big( - 1 -2x - 2\sum_{i = 1}^\tau | y_i | \Big)\exp(-f(x)) \quad \text{if} \quad x > 0 \,.
		\end{cases}
	\end{equation*}	
	and
	\begin{equation*}
		g''(x) =
		\begin{cases}		 
			&\bigg(\Big( - 1 -2x + 2\sum_{i = 1}^\tau | y_i | \Big)^2 - 2\bigg)\exp(-f(x)) \quad \text{if} \quad x < 0\\
			&\bigg(\Big( - 1 -2x-2\sum_{i = 1}^\tau | y_i | \Big)^2 - 2\bigg)\exp(-f(x)) \quad \text{if} \quad x > 0 \,.
		\end{cases}
	\end{equation*}
	Note that $g''(x) < 0 $	for all $x$,  as 
	\begin{equation*}
		1 + 2x + 2 \sum_{i  =1}^\tau |y_i| \leq 1 + \frac{2}{10} + \frac{2}{10}  \leq \sqrt{2} \, .
	\end{equation*}
	Finally,  $\lim_{x \to 0^{-}} g'(x) \geq \lim_{x \to 0^+}g'(x)$, concluding the proof of the concavity.
\end{proof}

\begin{lemma} \label{lem:convex conjugate potential}
	Suppose $L > \sqrt{2(V+1)}$. Let $\Phi(L) = \nu \E_{\eta \sim \rho}[\exp(\eta L - \eta^2 V)]$ with $\mathrm d\rho(\eta) \propto \exp(-\eta^2)\mathrm d\eta$, where  $\eta \in [0, \frac{1}{20(G+\varepsilon)(2\tau + 1)}]$. If $L \leq \frac{1}{20(G+\varepsilon)(\tau + 1)} (V + 1)$ then for $u \geq 0$
	\begin{equation*}
		\begin{split}
			\Phi^\star(u)\leq \sqrt{8u^2 (V + 1)\ln\left(24u^2 (V + 1) \frac{\pi}{\nu^2} + 1\right)}.
		\end{split}
	\end{equation*}
	If $L \geq \frac{1}{20(G+\varepsilon)(2\tau + 1)} (V + 1)$ then
	\begin{equation*}
		\Phi^\star(u) \leq 44 
		u(G+\varepsilon)(\tau + 1)
		\bigg(\ln\big(44 u(G+\varepsilon)(\tau + 1) \big) - 1 + \frac{1}{2}\ln \Big(\frac{\pi}{\nu^2}\Big)\bigg).
	\end{equation*}
\end{lemma}

\begin{proof}
	The proof of this lemma is a slight variation of the proof of Lemma~8 of \citet{vanderhoeven2019user} and a similar result has been obtained by \citet{jun2019parameter}.
	The initial part of the analysis is parallel to the analysis of Theorem~3 by \citet{koolen2015second}. Denote by $B = V + 1$ and by $Z = \int_{0}^{\frac{1}{20(G+\varepsilon)(2\tau+1)}} \exp(-\eta^2)$. For $\eta \leq \etahat = \frac{L}{2B}$, the function $ \eta \mapsto \eta L - \eta^2 B$ is non-decreasing. Therefore, for $[v, \mu] \subseteq [0, \big(20(G+\varepsilon)(2\tau+1) \big)^{-1}]$ such that $\mu \leq \etahat$:
	\begin{equation*}
		\Phi(L) =  \nu \frac{1}{Z}\int_{0}^{\frac{1}{20(G+\varepsilon)(2\tau+1)}}\exp(\eta L - \eta^2B)d\eta
		\geq \frac{1}{Z} \exp(v L - v^2 B).
	\end{equation*}
	First suppose that $\etahat \leq \frac{1}{20(G+\varepsilon)(2\tau+1)}$. Take $v = \etahat - \frac{1}{\sqrt{2B}}$, which yields
	\begin{equation*}
		\Phi(L) \geq \frac{\nu}{Z}  \exp\left(\frac{L^2}{4B} - \frac{1}{2}\right) = g(m(L))
	\end{equation*}
	where $g(x) = \exp \big(x - 1/2 - \ln \left( Z / \nu \right)\big)$ and $m(x) =  x^2 / 4B$. By \citet[Theorem~2]{hiriart2006note} we have
	\begin{equation}\label{eq:gamma to minimize}
		\begin{split}
			\Phi^\star(u)\leq 
			g^\star \circ m^\star(  \u)  & = \inf_{\gamma \geq 0} g^\star(\gamma) + \gamma m^\star\left(\frac{u}{\gamma}\right) \\
			& = \inf_{\gamma \geq 0} \gamma \ln (\gamma) + \gamma\Big(\ln \Big(\frac{Z}{\nu}\Big) - \frac{1}{2}\Big) + \frac{1}{\gamma} 4 u^2 B.
		\end{split}
	\end{equation}
	Denote by $S = \ln( Z / \nu ) - \half$ and $H = 4 u^2 B$. Setting the derivative to $0$, we find that the value $\hat{\gamma} = \sqrt{\frac{2H}{W(2H\exp(2S + 2))}}$ minimizes \eqref{eq:gamma to minimize}, where $W$ is the Lambert function. Plugging $\hat{\gamma}$ in \eqref{eq:gamma to minimize} and simplifying by using $\ln(W(x)) = \ln(x) - W(x)$ for $x>0$ gives
	\begin{equation*}
		\Phi^\star(u)\leq \sqrt{2H W(2H\exp(2S + 2))} - \sqrt{\frac{2H}{W(2H\exp(2S + 2))}} \leq \sqrt{2H W(2H\exp(2S + 2))}.
	\end{equation*}
	Using $ W(x) \leq \ln(x+1)$ \citep[Lemma 17]{orabona2016coin} we obtain
	\begin{equation*}
		\Phi^\star(u)\leq \sqrt{2H\ln(2H\exp(2S + 2) + 1)} \leq \sqrt{8u^2 B\ln\left(24u^2 B \frac{\pi}{\nu^2} + 1\right)},
	\end{equation*}
	where we used that $Z \leq \sqrt{\pi}$.

	Now suppose that $\etahat > \frac{1}{20(G+\varepsilon)(2\tau+1)}$, which is equivalent to $10(G+\varepsilon)(2\tau + 1)L > B$ . Then
	\begin{equation*}
		\Phi(L) \geq \frac{\nu}{Z}  \exp\big(\left(v-v^2 10(G+\varepsilon)(2\tau + 1)\right)L\big).
	\end{equation*}
	The convex conjugate of $f(L) = c_1 \exp(c_2 L)$, where $c_1, c_2 > 0$, can be computed using standard properties of convex conjugates and is given by $f^\star(y) = \frac{y}{c_2} \ln\left(\frac{y}{c_1 c_2}\right) - \frac{y}{c_2}$ for $y \geq 0$ (see for example \citet[section 3.3]{boyd2004}), which can be use to upper bound $\Phi^\star(u)$ for $u \geq 0$ by the order reversing property of convex conjugates:
	\begin{equation*}
		\Phi^\star(u) \leq  \frac{u}{v-v^2 10(G+\varepsilon)(2\tau + 1)}\bigg(\ln\left(\frac{u}{v-v^2 10(G+\varepsilon)(2\tau + 1)}\right) - 1 + \frac{1}{2}\ln \left(\frac{\pi}{\nu^2}\right)\bigg),
	\end{equation*}
	where we used that $Z \leq \sqrt{\pi}$. Picking $v = \frac{5- \sqrt{5}}{200(G+\varepsilon)(2\tau + 1)}$ gives us:
	\begin{equation*}
		\Phi^\star(u) \leq  u44\big(G+\varepsilon)(2\tau + 1\big)
		\bigg(\ln\big(u44(G+\varepsilon)(2\tau + 1)\big) - 1 +  \frac{1}{2}\ln \left(\frac{\pi}{\nu^2}\right)\bigg),
	\end{equation*}
	which completes the proof.
\end{proof}

\subsection{Black-Box Reduction}\label{app:black-box}

\begin{algorithm}
	\caption{Black-Box Reduction}\label{alg:black box}
	\KwIn{``Direction'' algorithm $\mathcal{A}_\mathcal{Z}$ and ``scaling'' algorithm $\mathcal{A}_\mathcal{V}$}
	\For{$t = 1 \ldots T$}{
		~~Get $\z_t \in \mathcal{Z}$ from $\mathcal{A}_\mathcal{Z}$ \\
		Get $v_t \in \reals$ from algorithm $\mathcal{A}_\mathcal{V}$ \\
		Play $\w_t = v_t \z_t$ and receive $\hat{\g}_t$\\
		Send $\hat{\g}_t$ to algorithm $\mathcal{A}_\mathcal{Z}$\\
		Send $\langle \z_t, \hat{\g}_t \rangle$  to algorithm $\mathcal{A}_\mathcal{V}$
	}
\end{algorithm}

To derive a $d$-dimensional comparator-adaptive algorithm for a graph we will use the black-box reduction in Algorithm \ref{alg:black box}. As an alternative to the black-box reduction one could also run a variation of Algorithm \ref{alg:pwa} for $u \in \reals$ in each dimension, which would result in a regret bound similar to that of AdaGrad \citep{duchi2011adaptive}. One could also generalize Algorithm~\ref{alg:pwa} to a higher dimensional version by replacing the scalar $\eta$ with a vector and adjusting the feedback accordingly, but then no closed-form solution of the integral defining the predictions exists. The algorithm in this section is the black-box reduction presented in \citet{cutkosky2018black}. The guarantee of Algorithm \ref{alg:black box} can be found in Lemma~\ref{lem:black box reduction} below, whose short proof was originally given by \citet{cutkosky2018black} and is repeated below for completeness.

\begin{lemma}\label{lem:black box reduction}
	Let $\lregret_T^\mathcal{V}(\|\u\|) = \sumT (v_t - \|\u\|) \langle \z_t , \g_t \rangle $ be the regret for learning $\|\u\|$ by Algorithm $\Aset_\mathcal{V}$ and let $\lregret_T^\mathcal{Z}\big(\frac{\u}{\|\u\|}\big) = \sumT \langle \z_t - \frac{\u}{\|\u\|}, \g_t \rangle$ be the regret for learning $\frac{\u}{\|\u\|}$ by $\mathcal{A}_\mathcal{Z}$. Then Algorithm\ref{alg:black box} satisfies 
	\begin{equation*}
		\regret_T(\u) \leq \lregret_T^\mathcal{V}(\|\u\|) + \|\u\|\lregret_T^\mathcal{Z}\left(\frac{\u}{\|\u\|}\right).
	\end{equation*}
\end{lemma}
\begin{proof} \textbf{of Lemma \ref{lem:black box reduction}}
	By definition we have
	\begin{equation*}
		\begin{split}
			\regret_T(\u) 
			\leq   \sumT \langle \w_t - \u, \g_t \rangle  
			&= \sumT \langle \z_t, \g_t \rangle \big(v_t - \|\u\|\big) + \|\u\|\sumT\Big\langle \z_t - \frac{\u}{\|\u\|}, \g_t \Big\rangle  \\
			&= \, \lregret_T^\mathcal{V}\big(\|\u\|\big) + \|\u\| \lregret_T^\mathcal{Z}\Big(\frac{\u}{\|\u\|}\Big).
		\end{split} 
	\end{equation*}
\end{proof}

Note that in the regret bound of Algorithm \ref{alg:black box}
the regret of $\mathcal{A}_\mathcal{Z}$ scales with $\|\u\|$.
This means that we can use \eqref{eq:standard inexact idea} and
upper bound $\inner{\z_t - \frac{\u}{\|\u\|}}{\g_t - \hat{\g}_t}
\leq 2\varepsilon$ by using H{\"o}lder's inequality. In turn
this allows us to use any multiagent algorithm for
$\mathcal{A}_\mathcal{Z}$ with suitable guarantees. 

For example, in Theorem~\ref{th:delaycompddim}, 
we detail the regret bound obtained when using the multiagent algorithm of \citet{hsieh2020multiagent}, which is delay-tolerant. The following results lead up to the proof of Theorem~\ref{th:delaycompddim}.

\begin{proposition}[Proposition~9 from \citet{hsieh2020multiagent}]\label{prop:adadelaydist}
	The Ada-Delay-Dist algorithm for Online Learning with Delays bounded by $D(\mathcal G)$ on the unit sphere satisfies
	\begin{equation*}
		\regret_T	\leq 4 \sqrt{\Lambda_T}  + 6 G D(\mathcal G) \, . 
	\end{equation*}
\end{proposition}
We account for the inexact gradients in an elementary way, by just adding up the error. This coarse treatment of the error is sufficient for our purpose.
\begin{corollary}[AdaDelay-Dist with approximate gradients]\label{cor:inexact-ada-delay-dist}
	Under the assumptions of Proposition~\ref{prop:adadelaydist}, 
	when given approximate gradients as input $\hat \g_t$ such that $\| \hat \g_t - \g_t \|\leq \eps$, AdaDelay-Dist enjoys the bound
	\begin{equation*}
		\regret_T	\leq
		4  \sqrt{\Lambda_T + 9\, \eps G D(\Graph) T} 
		+ 2 \eps T + 12 GD(\Graph) \, . 
	\end{equation*}
\end{corollary}
Corollary~\ref{cor:inexact-ada-delay-dist} is proved directly thanks to the following lemma that bounds on the lag on the approximate gradients, and using the fact that approximate gradients are bounded by $G + \eps \leq 2G$.
\begin{lemma}
	If approximate gradients are such that $\| \hat \g_t - \g_t \| \leq \eps$, then
	\begin{equation*}
		\widehat \Lambda_T \leq \Lambda_T + 9\eps G D(\Graph)  T \, . 
	\end{equation*}
\end{lemma}
\begin{proof}
	For any $t$ and $s$ in $[T]$, 
	\begin{equation*}
		\|\hat \g_t \|\| \hat \g_s \| 
		\leq \big(\|\g_t \| + \eps\big) \big(\| \g_s \| + \eps \big) 
		\leq \| \g_t \|\| \g_s \| + 2G\eps + \eps^2 
		\leq \| \g_t \|\| \g_s \| + 3G\eps \, . 
	\end{equation*}
	Thus, as there are less than $D(\Graph) $ terms in the sum, 
	\begin{equation*}
		\| \hat \g_t  \|^2 + 2 \|\hat \g_t \|\sum_{s \in \gamma(t)} \| \hat \g_s \|
		\leq
		\|  \g_t  \|^2 + 2 \| \g_t \|\sum_{s \in \gamma(t)} \|  \g_s \|
		+ 3G\eps + 6 D(\Graph) G \eps
	\end{equation*}
	We get the claimed result by upper bounding $3G\eps + 6 D(\Graph) G \eps \leq 9 G D(\Graph) \eps$ and summing over $t \in [T]$.
\end{proof}

\begin{theorem}\label{th:delaycompddim}
	Suppose that $\|\g_t\| \leq G$ and that $\|\hat{\g}_t - \g_t\| \leq \varepsilon$, and $\mathcal{A}_\mathcal{Z}$ is AdaDelay-dist. Using Algorithm \ref{alg:pwa} as $\mathcal{A}_\mathcal{V}$ guarantees that Algorithm \ref{alg:black box} satisfies
	\begin{equation*}
		\regret_T(\u) \leq  \nu + \|\u\| \mathcal B(T) \, , 
	\end{equation*}
	where
	\begin{multline*}
	 \mathcal B(T) = 
	 	 4 \eps T + \sqrt{8 \Big(\Lambda_T + 24\eps G D(\Graph) T + 1\Big)
	 	 	\ln_+\! \Big(   \frac{2036 \,  \|\u\|^2 D(\Graph) G^2 T}{\nu^2}\Big)}  \\
		+ 4 \sqrt{\Lambda_T + 9\, \eps G D(\Graph) T} 
		+ 276\,G D(\Graph)  \ln_+ \! \Big(\frac{312\, \|\u\|G \, D(\Graph)}{\nu} \Big)  \,  . 
	\end{multline*}
\end{theorem}
The proof is merely a combination of the decomposition of regret (Lemma~\ref{lem:black box reduction}), and algorithm guarantees.
More generally, similar bounds can be derived as long as the algorithm $\mathcal A_{\mathcal Z}$ for learning the direction is delay-tolerant, by adapting Proposition~\ref{prop:adadelaydist} with different numerical constants.
\begin{proof} \textbf{of Theorem~\ref{th:delaycompddim}}
	By Lemma \ref{lem:black box reduction} and the guarantee of $\mathcal{A}_\mathcal{Z}$ from Corollary~\ref{cor:inexact-ada-delay-dist}, we have 
	\begin{equation*}
		\begin{split}
			\regret_T(\u) \leq & \lregret_T^\mathcal{V}(\|\u\|) + \|\u\|\lregret_T^\mathcal{Z}\Big(\frac{\u}{\|\u\|}\Big) \\
			\leq & \lregret_T^\mathcal{V}(\|\u\|) 
			+ \|\u\|
			\Big(4 \sqrt{\Lambda_T + 9\, \eps G D(\Graph) T} 
			+ 2 \eps T + 12 GD(\Graph) \Big).	
		\end{split}
	\end{equation*}
	Importing the bound for the scale learning (Theorem~\ref{th:pwa}), we get
	\begin{align*}
		\lregret_T^\mathcal{V}(\|\u\|)  \leq \nu& + 2\|\u\|T\varepsilon + \|\u\| \max \Bigg\{\,264\,G D(\Graph) 
		\ln_+ \! \Big(\frac{312\, \|\u\|G \, D(\Graph)}{\nu} \Big) \, , \\ 
		& \sqrt{8 \Big(\Lambda^h_t + 24 \eps G D(\Graph) T + 1\Big)
			\ln_+\! \Big(   \frac{2036 \,  \|\u\|^2 D(\Graph) G^2 T}{\nu^2}\Big)} \Bigg\}.
	\end{align*}
	Note finally that $|h_t| = |\langle z_t, \g_t \rangle |\leq \| \g_t \|  $ so that
		$\Lambda^h_T \leq \Lambda_T$. The claimed result follows through standard boundings. 
\end{proof}

\subsection{Proof of Theorem~\ref{th:fullstochencoding}}\label{app:proof fullstoch}
Our bound is actually stronger than stated in Theorem~\ref{th:fullstochencoding}. As can be seen from the proof, the bound
could also be stated in terms of $\smash{\hat{\Lambda}_T}$, which is
$\Lambda_T$ with $\hat{\g}_t$ instead of $\g_t$. This means that the
bound scales with the effective gradient range $\max_t\|\g_t\|$ rather
than the worst-case bound on the gradients. Similarly, we have presented
the bound as if the learner suffers the maximum delay $D(\Graph)$ each
round, but the bound could also be stated in terms of the effective
delay $|\gamma(t)|$. For simplicity, we only present the worst-case regret bound.

We recall that the algorithm used is the black-box reduction applied to the gradients encoded with sparsified quantization, with AdaDelay-dist (from \citet{hsieh2020multiagent}) as $\mathcal{A}_\mathcal{Z}$ and
Algorithm~\ref{alg:pwa} as $\mathcal{A}_\mathcal{V}$. Algorithm~\ref{alg:pwa} is tuned with $\eps = 0$ and the upper bound $2dG$ on
$\|\hat \g_t \|$.
\thfullstochencoding*
\begin{proof}
	First, by convexity of $\ell_t$ and using that $\E_t[\hat{\g}_t] = \g_t$ we have that 
	\begin{equation*}
		\E[\regret_T(\u)] \leq \E\left[\sumT\inner{\w_t - \u}{\hat{\g}_t}\right].
	\end{equation*}
	We have that $\max_t \|\hat{\g}_t\| \leq 2dG$ by Theorem~\ref{th:stochasticupper}. Let us now call to the detailed regret bound for our algorithm, that is, Theorem~\ref{th:delaycompddim} used
    with the estimated losses $\{\hat \g_t \}$ and with
    $\varepsilon = 0$.
    Note that in this case, Theorem~\ref{th:delaycompddim} is applied \emph{directly} to the approximate gradients, which makes the choice $\eps= 0$ valid. We find
	\begin{equation*}
		\sumT\inner{\w_t - \u}{\hat{\g}_t}
		\leq \nu + \|\u  \| \,  \hat{\mathcal B(T)} \,  , 
	\end{equation*}
	where
	\begin{align*}
		\hat {\mathcal B(T)}  
		 := &  \sqrt{8 \Big(\hat \Lambda_T + 1\Big)
			\ln_+\! \Big(   \frac{2036 \,  \|\u\|^2 T(2dG)^2}{\nu^2}\Big)} 
		 + 4 \sqrt{\hat \Lambda_T}\\
		& + 276\,(2dG) D(\Graph)  \ln_+ \! \Big(\frac{312\, \|\u\|2dG \, D(\Graph)}{\nu} \Big) \\
		\leq & 
		\sqrt{47 \Big(\hat \Lambda_T + 1\Big)
			\ln_+\! \Big(   \frac{8144 \,  \|\u\|^2 Td^2\,G^2}{\nu^2}\Big)} 
		\\
		& + 552\,dG D(\Graph)  \ln_+ \! \Big(\frac{624\, \|\u\|dG \, D(\Graph)}{\nu} \Big) \, .
	\end{align*}
	Now let us upper bound $\E\! \big[ \hat \Lambda_T \big]$. Define the short-hand notation
	\begin{equation*}
		\alpha = \frac{2d}{m} 
		\quad \text{and} \quad
		\beta = \frac{1}{m} G^2  
		\quad \text{with} \quad
		m = \lfloor k / (3 \lceil \log_2 (d) \rceil  + 2) \rfloor.
	\end{equation*}
	Then
	\begin{equation*}
		\E\big[ \|\hat \g_t \|^2  \big]
		 = \E\big[ \|\hat \g_t   -  \g_t \|^2  \big] + \|\g_t \|^2
		\leq  (\alpha+1)\| \g_t \|^2 + \beta \,  .
	\end{equation*}
	Using Jensen's inequality and the guarantees of Theorem~\ref{th:stochasticupper} we find
	\begin{align*}
		\E\big[ \|\hat \g_t \| \| \hat \g_s \|  \big]
		\leq & \sqrt{\E\!\big[ \|\hat \g_t \|^2\big] \E\!\big[ \| \hat\g_s \|^2  \big]  }\\
		\leq & \sqrt{ (\alpha+1) \| \g_t\|^2  + \beta } \, \sqrt{ (\alpha+1)  \| \g_s\|^2 + \beta } \\
		\leq & (\alpha+1) G^2  + \beta.
	\end{align*}
	Now, replacing $\alpha$ and $\beta$ by their values we see that 
	\begin{align*}
		\E\big[ \|\hat \g_t \| \| \hat \g_s \|  \big] \leq G^2\Big(1 + \frac{d+1}{m}\Big).
	\end{align*}
	We sum these inequalities in the expression of $\hat \Lambda_T$ and use that $|\gamma(t)| \leq D(\Graph)$ to get
	\begin{align*}
		\E \big[ \hat \Lambda_T\big] \leq &
			G^2T(1 + D(\Graph))\left(1 + \frac{d+1}{m}\right) \leq G^2T(1 + D(\Graph))\left(1 + \frac{(d+1)(3\log_2(d) + 3)}{k}\right) \, . 
	\end{align*}
	which, after replacing $\hat{\Lambda}_T$ in $\hat {\mathcal B(T)} $ and setting $k = D(\Graph) / b$ completes the proof.
\end{proof}

\section{Details of Section \ref{sec:partitionlearning}}\label{app:partition learning}

We denote by $\lregret_{\F}(\u) = \sum_{t: I_t \in \F} \inner{\w_t^\F - \u}{\g_t}$ the linearised regret in graph $\F$. 

\begin{lemma}\label{lem:partition}
	Let $\Qset $ be a collection of subgraphs of $\mathcal G$ and suppose for each subgraph $\mathcal F \in \Qset$ $\lregret_{\F_j}(\0) \leq \nu$, discarding any message older than $D_{\Qset} = \max_{\mathcal{F} \in \Qset}D(\mathcal{F})$ rounds. Then, playing $\w_t = \sum_{\mathcal F \in \Qset} \w_t^{\mathcal F} \mathds{1} \{ I_t \in \mathcal F \}$ simultaneously guarantees for any $\Qset$-partition $\{\mathcal F_1, \dots, \mathcal F_r \}$ and for any $\u_1 \dots, \u_{r} \in \reals^d$ that
	\begin{align*}
		\sumT \ell_t(\w_t) - \sum_{ j = 1}^r \sum_{\substack{t: I_t \in \mathcal F_j}} \ell_t\left(\u_j\right) 
		\leq &  |\Qset|\nu + \sum_{ j = 1}^r \lregret_{\F_j}(\u_j)  \, . 
	\end{align*}	
\end{lemma}
\begin{proof}
	We start by upper bounding the regret by its linearized version:
	\begin{equation*}
		\sumT \ell_t(\w_t) - \sum_{ j = 1}^r \sum_{\substack{t: I_t \in \mathcal F_j}} \ell_t(\u_j)  
		= \sum_{j = 1}^r\sum_{t: I_t \in \mathcal F_j} \left(\ell_t(\w_t) - \ell_t\left(\u_j\right) \right)
		\leq  \sum_{j = 1}^r\sum_{t: I_t \in \mathcal F_j} \inner{\w_t - \u_j}{\g_t} \, .
	\end{equation*}
	Denote by $\Pset = \{\F_1, \ldots, \F_r\}$ an arbitrary $\Qset$-partition. We rewrite the bound above, replacing $\w_t$ by its value, and doing some manipulations on the indices,
	\begin{align*}
		\sumT \inner{\w_t}{\g_t} - \sum_{j = 1}^r\sum_{t: I_t \in \mathcal F_j} \inner{\u_j}{\g_t} 
		& =  \sumT \sum_{\mathcal F: I_t \in \mathcal F} \inner{\w_t^{\mathcal F} }{\g_t} - \sum_{j = 1}^r\sum_{t: I_t \in \mathcal F_j} \inner{\u_j}{\g_t}  \\
		&\; =  \sum_{\mathcal F \in \Qset \setminus \Pset} \lregret_\F(\0)  + \sum_{ j = 1}^r \lregret_{\F_j}(\u_j) \\
		& \;  \leq |\Qset|\nu + \sum_{ j = 1}^r \lregret_{\F_j}(\u_j) \, ,
	\end{align*}
	where we used that $\lregret_{\mathcal F}(\0) \leq \nu$ for all $\mathcal F$.
\end{proof}

\thpartidet*

\begin{theorem}\label{th:partistoch}
	Suppose that $b \geq D_\Qset \, ( 3\lceil \log_2(d) \rceil+ 2)$. Let $\Qset $ be a collection of subgraphs of $\mathcal G$.
        Suppose the learner uses the algorithm of Theorem~\ref{th:delaycompddim} with $\varepsilon = 0$, the upper bound $2dG$ on $\|\hat \g_t \|$, and the stochastic encoding of Theorem~\ref{th:stochasticupper}, for each subgraph $\mathcal F
        \in \Qset$, discarding any message older than $D_{\Qset} =
        \max_{\mathcal{F} \in \Qset}D(\mathcal{F})$ rounds. Then,
        setting $\nu = 1 / |\Qset|$, $k = \lfloor b / D_\Qset \rfloor$, and playing $\w_t =
        \sum_{\mathcal F \in \Qset} \w_t^{\mathcal F} \mathds{1} \{ I_t
        \in \mathcal F \}$ guarantees that
	\begin{align*}
		&\E\left[\sum_{j=1}^r \regret_{\F_j}(\u_j)\right] =  \O\Big(\sum_{j=1}^r
		\|\u_j\|G\sqrt{\left(1 + \frac{d D_\Qset}{b}\right)D(\F_j)T_j\ln \big(1 + |\Qset| D(\F_j) \|\u_j\|T_jG\big)} \; \Big) \,, 
	\end{align*}
        for any $\Qset$-partition $\{\mathcal F_1,
        \dots, \mathcal F_r \}$ and for any $\u_1 \dots, \u_{r} \in
        \reals^d$.
\end{theorem}

The proofs of Theorems \ref{th:partidet} and \ref{th:partistoch} follow from applying Lemma~\ref{lem:partition}, and a slight modification of the proofs of Theorems \ref{th:fulldetencoding} and \ref{th:fullstochencoding}, in which we can use that $k = \lfloor b / D_\Qset \rfloor$.

\subsection{Example Collections of Subgraphs}\label{sec:example collection graphs}

In this section we provide an example collection of subgraphs. For each node $n$, take all nodes that are with in distance $\alpha(n) = 0, 1, \ldots, E(n)$, where $E(n)$ is the eccentricity of node $n$, and form them into a subgraph. The collection of this set of subgraphs is has cardinality $|\Qset| = |\Nset| + \sum_{n \in \Nset} E(n) \leq |\Nset|(1 + |\Nset|)$. Alternatively, one may take $\alpha(n) = 2^{\omega(n)}$ for $\omega(n) = 1, \ldots, \lfloor\log_2(E(n))\rfloor$ to approach the regret bounds of the previously defined $\Qset$ within a factor 2. Both collections can be used to adapt to the scenarios described in Figures \ref{fig:galaxiesFarFarAway} and \ref{fig:embeddedGalaxies} since the clusters may will be contained in one of the subgraphs.

\section{Details on the Encoding Schemes}

In this section, we provide the detailed presentation and analysis of the encoding schemes referred to in Section~\ref{sec:limited_com}, eventually proving Theorem~\ref{th:stochasticupper}. We also discuss briefly some relevant literature.

\subsection{(Two) Determinisitic encoding Schemes}\label{app:deterministic_enc}
We discuss two possible encoding schemes. The first encoding scheme is a non-constructive encoding using covering numbers. Given well-known results on the covering number of the ball, there exists a covering of the ball $\mathcal B_2(G)$ with $2^k$ balls of radius $3 \cdot 2^{-k/d}G$. Therefore, for any vector $\x \in \mathcal B_2(G)$, one can transmit $\widehat \x$ the center of a ball of the covering it belongs to, using $k$ bits.
Doing so, we have an encoding which ensures 
$
\| \x - \widehat \x \| \leq 3 \, \cdot  2^{-k / d} G \, .
$
Note that this implies in particular that $ \| \widehat \x \| \leq G(1 + 3 \cdot 2^{-k/d}) \leq 4G$. 

The drawback of this first encoding scheme is that it requires an explicit optimal covering of the ball. The second encoding scheme we provide is slightly worse in terms of error, but is more practical due to its simplicity. 

In the second encoding scheme, the approach is to send each coordinate separately. Although this encoding scheme is suboptimal in terms of error, it has the advantage of being straightforward to implement.
Using $\lfloor k/d \rfloor$ bits per coordinate (we assume here that $k \geq d$), and sending each coordinate separately yields for all coordinates $i$,
\begin{equation*}
	|x_i - \widehat x_i | \leq 2^{-k/d + 2} G,
	\quad \text{thus} \quad
	\| \x - \widehat \x  \|\leq \sqrt{d} \, 2^{-k / d + 2} G \, . 
\end{equation*}
Compared to the theoretical bound from covering numbers, we lose a $\sqrt d$ factor in the error bound. However, the practicality gain can make this choice worthwhile when $k/d$ is sufficiently large.

\subsection{Stochastic Encoding}\label{sec:details stoch encoding}

To stochastically encode an arbitrary vector $\x \in
\mathcal B_2(G) \subset \reals^d$, we will combine $p$-level quantization
with sparsification, which are defined as follows:

\paragraph{$p$-Level Quantization of a Single Coordinate}

Given a precision parameter $p \in \N$, the $p$-level quantization of 
any coordinate $x_i$ of $\x$ is
\begin{equation*}
	\tilde x_i = \text{sign}(x_i) 2^{-p}G \Big\lfloor \frac{2^{p}
        x_i}{G} \Big\rfloor + 2^{-p}G \, b_i.
\end{equation*}
The first term in this expression is deterministic and essentially corresponds to
the truncation of $x_i/G$ to the first $p$ digits in its binary
expansion. The second term is random, with $b_i \in \{0,1\}$ a Bernoulli
variable with $\Pr(b_i = 1) = (2^p/G) (x_i -  \text{sign}(x_i) 2^{-p}G \lfloor {2^{p}x_i}/{G} \rfloor )$ chosen to make $\tilde x_i$ an
unbiased estimate of $x_i$. The essential properties of the $\tilde x_i$ are that
\begin{equation*}
	\E [\tilde x_i] = x_i 
	\quad \text{and} \quad 
	(x_i - \tilde x_i)^2 \leq 2^{-2p} G \, . 
\end{equation*}
All in all, this randomized encoding requires $1$ bit to encode the sign
of $x_i$, and $p$ bits for the deterministic part, and $1$ bit for $b_i$, so
$p+2$ bits in total.

\paragraph{Sparsification}

Instead of encoding all coordinates of $\x$, we sample a single
coordinate $i_S$ uniformly at random from $\{1, \dots, d\}$ and transmit
only $\tilde{x}_{i_S}$. At the decoder, this results in a sparse vector
\[
  \widehat{\x} = d \tilde x_{i_S} \mathbf{e}_{i_S},
\]
where the factor $d$ is an importance weight that ensures that
$\E[\widehat \x] = \x$. Encoding $i_S$ requires $\lceil \log_2(d)
\rceil$ bits, so $p$-level quantization and sparsification together
require $\lceil\log_2(d)\rceil + p + 2$ bits.

\paragraph{Sparsified Quantization}

To reduce the variance, we repeat the construction above $m$ times and
average the resulting vectors $\widehat \x_1, \ldots, \widehat \x_m$ to
obtain our final estimate
\[
  \widehat{\x} = \frac{1}{m} \sum_{j=1}^m \widehat \x_j \, .
\]
We call $\widehat \x$ the sparsified quantization of $\x$ with precision
$p$ and number of repetitions $m$. It can be communicated with
$m(\lceil\log_2(d)\rceil + p + 2)$ bits, and satisfies the following
property:
\begin{proposition}
	For any $\x \in \mathcal B_2(G)$, the sparsified quantization
        $\widehat \x$ of $\x$ with $m=1$ repetition satisfies
	\begin{equation*}
		\E\!\big[ \widehat \x  \big] = \x,
		\quad \quad
		\E\big[ \|\x - \widehat \x \|^2 \big] \leq 2 d \|\x \|^2 + d^2 \,  2^{-2p} G^2 
	\end{equation*}
        and 
        \[
        \| \widehat \x \| \leq d(\| \x \|_{\infty} + 2^{-p}G)\leq 2dG.
        \]
	In particular, for $p = \lceil \log_2 d \rceil$, the expected squared error is less than $2 d \| \x \|^2 + G^2$ and the number of bits communicated is less than $3\lceil \log_2 d \rceil +2$. 
\end{proposition}

\begin{proof}
	The unbiasedness is straightforward from the independence between $\tilde x_i$ and $i_S$:
	\begin{equation*}
		\E\, [\widehat \x] 
		= d \sum_{i = 1}^d \P[i = i_S]  \E [\, \tilde x_i \mathbf{e}_i   ]
		= d \sum_{i = 1}^d  \frac{1}{d} \, x_i \, \mathbf{e}_i  
		= \x \, .
	\end{equation*}
	Let us compute the square-error for a fixed coordinate $i \in \{1, \dots, d\}$,
	\begin{multline*}
		\E \big[  (x_i  - \widehat x_i)^2 \big]
		= \E \big[  (x_i  - \widehat x_i)^2 \mathds 1 \{ i = i_S \} \big]
		+ x_i^2  \, \P [i \neq i_S ] \\
		= \E \big[  (x_i  - d \tilde x_i)^2] \, \P [ i = i_S]
		+ x_i^2  \, \P [i \neq i_S ]
		=  \E \big[  (x_i  - d \tilde x_i)^2] \frac{1}{d} + x_i^2 \Big(1  - \frac{1}{d} \Big) \, 
	\end{multline*}
	where we used the fact that $\tilde x_i$ is independent from $i_S$.
	Moreover, using the unbiasedness of $\tilde x_i$,
	\begin{equation*}
		\E\!\big[  (x_i  - d \,  \tilde x_i)^2 \big]
		= \E \big[  \big(x_i  - d \,  x_i + d(x_i - \tilde x_i )\big)^2 \big]
		= ( d- 1)^2 x_i^2 + d^2 \E \big[ (x_i - \tilde x_i)^2  \big]
		\leq d^2( x_i^2 + 2^{-2p} G^2) \, . 
	\end{equation*}
	Finally, by summing over the coordinates,
	\begin{equation*}
		\E\!\big[ \|\x - \widehat \x \|^2 \big]
		\leq d^2 \frac{1}{d} \sum_{i = 1}^d \big(x_i^2 + 2^{-2p}G^2 \big)+ \left(1 - \frac{1}{d}\right)\sum_{i = 1}^d  x_i^2
		\leq 2 d \|\x \|^2 + d^2 \,  2^{-2p}G^2 \, ,
	\end{equation*}
	concluding the proof.
\end{proof}
Adding repetitions reduces the variance of any unbiased stochastic
encoding:
\begin{proposition}[Combining Encodings]
	If $\widehat \x_1, \dots, \widehat \x_m$ are independent unbiased encodings of a vector $\x \in \mathcal B_2(G)$, then the error of the average encoding scheme is
	\begin{equation*}
		\E \bigg[ \Big\| \x - \frac{1}{m}\sum_{i= 1}^m  \widehat \x_i  \Big\|^2 \bigg] 
		= \frac{1}{m^2} \sum_{i =1}^m \E \Big[ \| \x  - \widehat \x_i \|^2 \Big] \, .
	\end{equation*}
\end{proposition}
This is proved by developing the squared-norm.

In particular, sparsified quantization with precision $p = \lceil \log_2
d \rceil$ and $m = \lfloor k / (3\lceil \log_2 (d)\rceil+2) \rfloor$
repetitions satisfies the claim of Theorem~\ref{th:stochasticupper}.

\paragraph{Commonly used Stochastic Encodings.} The popular schemes of \citet{alistarh2017qsgd} and \citet{wen2017terngrad} use varying-length coding, and have results depending on the number of bits used to represent a floating-point number. Moreover, these results essentially assume that the number of bits per gradient is at least equal to the dimension $b = \Omega(d)$; the same restrictions apply to \citet{albasyoni2020optimal}, \citet{faghri2020adaptive}.
With a different formalism, \citet{stich2018sparsified}, \citet{shi2019a-distributed} use sparsification with no quantization and do not impose a constraint on the number of bits sent. \citet{acharya2019distributed} propose a sparsification scheme specific to encoding probability vectors, and \citet{mayekar2020limits} generalize it to other norm balls.
A number of schemes assume shared randomness between the agents,  e.g.,
\citet{mayekar2020ratq}, \citet{suresh2017distributed}. This is not
appropriate for our setting, because shared randomness can only be
achieved without communication by using pseudo-randomness based on a
shared seed. But then the adversary might guess the seed or otherwise
exploit the lack of real randomness to predict the randomness in the
stochastic encodings.%

\section{Lower Bounds}\label{app:lower-bounds}

We start by mapping out how each specific characteristic of our setting influences the regret, and how we can import known lower bounds from other settings.

\subsection{Overview}

As discussed at the start of Section~\ref{sec:limited_com}, we restrict
attention to algorithms that send messages containing approximate
gradients, with a limit of $k = \lfloor b / D(\Graph) \rfloor$ bits per
message.

Nodes only have access to approximations $\hat \g_t$ of the gradients
$\g_t$ that are computed at other nodes, but they have full access to
$\g_t$ for the subset of rounds in which they are active themselves. To
simplify the treatment of lower bounds, we restrict attention to
\emph{gradient-oblivious} methods, which only use $\hat \g_t$ at all the
nodes, even if a node could have used $\g_t$ instead. Although this
assumption restricts the class of algorithms, we believe it is a minor
restriction. For instance, all algorithms considered in the rest of the
paper satisfy it, and for large networks, in which the same node is
activated only a small number of times, the difference seems minor.

\paragraph{Comparison with Memory-Limited OCO} 

There are tight relations to the $(k, 1, 1)$ distributed online protocol
defined by \cite{shamir2014fundamental}. Under this protocol, an agent only has access to a collection of $k$-bit messages stored in its memory, each message encoding one loss function it has received. (The two ones in $(k, 1, 1)$ correspond to other parameters in the reference.)

In fact, for gradient oblivious algorithms that use $k$ bits per
gradient, in the special case where the same node is selected at every
time step, our setting simplifies to the $(k,1,1)$ setting. Therefore,
for any activation sequence, the game is at least as hard as $(k, 1,1)$,
as the active agent only has less information available. In particular, lower bounds for the $( k, 1, 1)$ setting automatically hold in our setting.

\paragraph{Deterministic vs. Stochastic Encodings}

We distinguish between deterministic and stochastic encodings.
Deterministic encodings only give non-trivial guarantees when
$k$ is at least of the same order as the dimension $d$
(Theorem~\ref{th:lower_boundI}). Stochastic encodings can still work in
the regime that $k \ll d$, as long as $k \gg d/T$
(Theorem~\ref{th:lower_boundII}). We will prove the lower bound for
stochastic encodings by reduction from a previous lower bound by
\cite{mayekar2020ratq}.

\paragraph{The Network Causes Delays} Because gradients take time to be
transmitted through the network, there is a direction connection to the
lower bound for Online Learning with Delays \cite{zinkevich2009slow}.
This effect is combined with the memory limits described above.
Informally, we can say that if $B( k, T)$ is a lower bound for the
memory-limited setting, then, under the activation sequence that
maximizes delays, a lower bound of $D(\mathcal G) B\big( k, T /
D(\Graph)\big)$ holds for our setting. This statement is made precise in
Lemma~\ref{lem:delays_lower_bound}.

\subsection{Proofs of Lower Bounds and Intermediate Results}

In the remainder of this section, we prove
Theorems~\ref{th:lower_boundI} (deterministic encodings) and
\ref{th:lower_boundII} (stochastic encodings). We first establish lower
bounds for the $(k, 1, 1)$ protocol separately for the deterministic and
stochastic case, based on the connection to the memory-limited setting.
We then exploit the connection to the online learning with delays
setting to turn these lower bounds for the $(k, 1, 1)$ protocol into the
lower bounds of the theorems.

\subsubsection{Limited Memory I: Deterministic Encodings}\label{app:lower_bound_det}

As discussed earlier, the decentralized online learning setting for
gradient-oblivious algorithms with $k$ bits per gradient is harder than
the memory-limited $( k, 1, 1)$ setting. In this section, we provide
lower bounds that are based on this connection.

The limiting factor with deterministic encodings is that there is finite
resolution, making some vectors indistinguishable from each other once
encoded. Formally, let $C: B_2(G) \to \{0,1\}^k$ be any deterministic
encoding scheme. Then we define the \emph{resolution} $\eps$ of $C$ to
be the maximum distance between two vectors encoded to the same message,
i.e., $\eps = \sup \big\{ \|\g - \mathbf h \| : C(\g)  = C(\mathbf h)
\big\}$. We can lower bound the resolution of $C$ using the covering
numbers of $\mathcal B_2(G)$ as follows:
\begin{lemma}
	Any deterministic $k$-bit encoding scheme on $\mathcal B_2(G)$ has
        resolution at least $\eps \geq 2G \, 2^{- k / d}$.
\end{lemma}
\begin{proof}
	The family of sets $C^{-1}(\mathbf m)$ for $\mathbf m \in \{0,
        1\}^k$ forms a partition of $\mathcal B_2(G)$, and each of these sets is contained in a Euclidean ball of diameter $\eps$ by definition of the resolution.
	Thus the balls form an $(\eps/2)$-cover of $\mathcal B_2(G)$.
        Since covering $\mathcal B_2(G)$ with balls of radius
        $\eps/2$ requires at least $(2G / \eps)^d$ balls
        \citep[Appendix~C]{GhosalVanDerVaart2017}\footnote{This can be
        shown by noting that
        the sum of the volumes of the balls in the cover must exceed the
        volume of $\mathcal B_2(G)$.}, it follows that $2^k = |\{0,
        1\}^k| \geq (2G / \eps)^d$, from which the result follows.
\end{proof}
Knowing that two gradient values that are $\eps$ apart are encoded the same way, an adversary can design a hard loss sequence against which a player would suffer linear regret in the $(k, 1, 1)$ setting.
\begin{lemma}[Deterministic Encoding with Limited Memory]\label{lem:det_lower_bound}
	In the $( k, 1, 1)$ protocol, for any algorithm using a
        deterministic encoding with resolution $\eps>0$, %
	for any $U > 0$%
	\begin{equation*}
		\sup_{G-\text{Lipschitz losses}} \;  \sup_{\u : \,  \|\u \| = U} \regret_T(  \u) \geq \frac{1}{4}\eps \,  T  \, \| \u \|  \, .
	\end{equation*}
\end{lemma}

\begin{proof}
	Let $\g$ and $\mathbf h $ be two vectors with same the encoding
        such that $\| \g - \mathbf h \| \geq \eps$. In the following, we
        consider exclusively linear losses $\ell_t : \u \mapsto \langle
        \g_t, \u \rangle$, which are fully specified by their gradient~$\g_t$. Consider an i.i.d. sequence of losses where at every time step $\g_t = \g$ with probability $1/2$ and $\g_t = - (\g + \mathbf h) / 2$ with probability $1/2$.
	
	Then at all times $t$, since the loss $\g_t$ is generated independently from $\w_t$,
	\begin{equation*}
		\E [\g_t ] = \frac{1}{2} \, \g + \frac{1}{2}\Big( -  \frac{\g + \mathbf h}{2} \Big) = \frac{\g - \mathbf h}{4}
		\quad \text{and} \quad
		\E\big[ \langle \g_t,  \w_t\rangle \, |\,  \w_t  \big]  = \Big\langle  \frac{\g - \mathbf h}{4}, \w_t\Big\rangle \, .
	\end{equation*}
	Therefore, for any comparator $\u \in \mathcal W$, the expected value of the regret against $\g_{1:T}$ is
	\begin{equation*}
		\E \!\big[\regret_T(\u ; \g_{1:T}) \big]
		= \E \bigg[  \Big\langle  \frac{\g - \mathbf h}{4}, \sum_{t = 1}^T \big(\w_t - \u\big) \Big\rangle  \bigg]
		= \Big\langle  \frac{\g - \mathbf h}{4}, \E \bigg[ \sum_{t = 1}^T \w_t  \bigg]- T\u \Big\rangle 
		\, ,
	\end{equation*}
	where we explicitely wrote the dependence on $\g_{1: T}$ for the sake of clarity.
	
	Now given a sequence $\g_{1:T}$ generated as above, define
        another sequence $\widetilde \g_{1:T}$ swapping $\g$ and $\mathbf h$; that is, where $\tilde \g_t = \mathbf h$ if $\g_t = \g$, and $\tilde \g_t = \g_t$ otherwise.
	As above, we also have the identity 
	\begin{equation*}
		\E \!\big[\regret_T\big(\u ; \widetilde \g_{1:T} \big) \big]
		= \Big\langle  \frac{\mathbf h -  \g}{4}, \E \bigg[ \sum_{t = 1}^T \w_t  \bigg]- T\u \Big\rangle  \,  .
	\end{equation*}
	Since the algorithm is gradient oblivious, and since the two loss sequences are encoded the same way, the algorithm cannot distinguish between a sequence $\g_{1:T}$ and its switched version $\widetilde \g_{1:T}$. Therefore the distribution of $\w_{1:T}$ is the same in both cases.
	In particular the expected value of the sum of the predictions $\E[\sum_{t = 1}^T \w_t]$ is the same under both loss distributions. 
	
	Therefore, by distinguishing cases on whether $ \langle \,  \E[\sum_{t = 1}^T \w_t] , \, \mathbf h - \g \rangle$ is non-negative or not, we observe that at least one of the following statement holds: either
	\begin{equation*}
		\E \!\big[\regret_T\big(\u ; \g_{1:T} \big) \big] 
		\geq  - \frac{T}{4} \langle \g - \mathbf h ,  \u\rangle
		\quad \text{for all } \u \in \mathcal W \, ,
	\end{equation*}
	or
	\begin{equation*}
		\E \!\big[\regret_T\big(\u ; \tilde \g_{1:T}) \big) \big] 
		\geq  \frac{T}{4} \langle \g - \mathbf h ,  \u\rangle  
		\quad \text{for all } \u \in \mathcal W \, . 
	\end{equation*}
	Let us conclude, for example, in the first case; the second case
        is completely symmetric. Picking $\u = \lambda(\mathbf h -
        \g)/\|\mathbf{h} - \g\|$
        for any $\lambda > 0$, we get
	\begin{equation*}
		\E \!\big[\regret_T\big(\u ; \g_{1:T} \big) \big] 
		\geq  \frac{T}{4} \| \g - \mathbf h\| \| \u \|
		\geq  \frac{T}{4} \eps \| \u \|\, . 
	\end{equation*}
	Since this holds in expectation for our distribution of losses, there exists a sequence $\g_{1:T}$ satisfying the claimed inequality.
\end{proof}
We also recall the comparator-adaptive lower bound from \citet{orabona2013dimension}. This lower bound holds in the standard OCO setting, and thus all the more so in the memory-limited setting. Together with Lemma~\ref{lem:det_lower_bound}, these two results yield Theorem~\ref{th:lower_boundI}, up to the impact of the delays.
\begin{theorem}[Theorem~2 in \citet{orabona2013dimension}]\label{thm:low_orab}
	In the standard OCO setting, fix an algorithm such that $\regret_T( \mathbf 0) \leq B$ for some number $B > 0$.
	For any $\u \in \R^d$, for $T$ large enough,
	\begin{equation*}
		\sup_{G-\text{Lipschitz losses}}
		\regret_T(\u)
		\geq 0.3 \, \| \u \| G\sqrt{T \ln \left(\frac{\| \u \| \sqrt{T} }{6 B }\right)} \, .
	\end{equation*}
\end{theorem}
Note that at first sight our statement seems stronger than that of the reference, as we transformed the ``there exists" quantifier into ``for any".
In fact, looking at the proof in the reference, the comparator is fixed to be $(U, 0, \dots, 0)$ and the sequence of (linear) losses takes values only in $ \{\pm ( G, 0, \dots, 0) \}$. This same construction can be made in the direction of any vector $\u \in \R^d$.

\subsubsection{Limited Memory II: Stochastic Encodings}\label{app:lower_stoch_enc}

The main downside of the deterministic encoding, namely the finite resolution, can be avoided using randomness, see Section~\ref{sec:stoch_enc}.
However, when there are less than $d$ bits available, even stochastic procedures have strong limitations. Indeed, \cite{mayekar2020ratq} show, via an optimization problem, that in the $(k , 1, 1)$ setting the limit on the number of bits appears as a constant in front of the minimax regret. 
\begin{theorem}[Corollary of Theorem~2 in \citet{mayekar2020ratq}]\label{thm:lower_bound_stoch_opt}
	In the $(k, 1, 1)$ protocol, for any algorithm, for any  $U > 0$
	\begin{equation*}
		\sup_{G-\text{Lipschitz losses}} \;  \sup_{\u : \,  \|\u \| = U}
		\E \big[ \regret_T( \u) \big]
		\geq c \, U G \sqrt{\max \Big(\frac{d}{ k}, 1 \Big) T} \quad ,
	\end{equation*}
	for some numerical constant $c$, and where the expectation is taken with respect to the randomness in the algorithm and the encoding.
\end{theorem}
In the mentioned reference, the setting refered to is Stochastic Optimization with access to a $k$-bit estimate of the gradient through an oracle. This is in fact equivalent to the $(k, 1, 1)$ setting, with only the objective differing. This lower bound follows directly from their result via a standard online-to-batch conversion; see e.g. \cite{hazan2016introduction}.

\subsubsection{Delays}
Let us now incorporate the impact of the delays induced by the transmission of information through the network. It consists in implementing an instance of online learning with delays in our decentralized setting.

Note also that the randomization cannot improve on the worst-case regret, as a randomized algorithm can always be converted into a deterministic algorithm incurring the same regret against linear losses. For example, such a conversion is obtained by considering the virtual algorithm playing the expected prediction.

\begin{lemma}[Reduction to Learning with Delays]\label{lem:delays_lower_bound}
	For any network with diameter $D = D(\Graph)$, there exists an activation sequence such that decentralized OCO for gradient-oblivious algorithms with $k$ bits per message is equivalent to online learning with delays of length $\lfloor D / 2 \rfloor$ and with $k$-bit memory.
	
	In particular, for any $U > 0$, if $B(k, t, U)$ is a lower bound on the minimax regret for the $(k, 1, 1)$ protocol in OCO against comparators $\u \in \mathcal B_2(U)$ for all $t \in \N$, then 
	\begin{equation*}
		\sup_{G-\text{Lipschitz losses}} \; 
		\sup_{\u : \,  \|\u \| \leq U}
		\E \big[ \regret_T( \u) \big]
		\geq \Big\lfloor \frac{D}{2}  \Big\rfloor B\Big(k , \Big\lfloor \frac{2T}{D} \Big\rfloor, 
		U\Big) \, , 
	\end{equation*} 
	where the expectation is taken with respect to the randomness in the algorithm and the encoding.
\end{lemma}
\begin{proof}
	Consider a maximal path through the graph of length $D$. Without loss of generality, we index the sequence of nodes by $\{1, \dots, D\}$, assuming that node $n$ is connected to node $m$ if and only if $|n - m| = 1$.
	
	Then consider the activation sequence that sequentially selects the nodes $1, 3, \dots, 2\lfloor D /2 \rfloor - 1$ and starts over to $1$. The information available to the active node $I_t$ is always $\lfloor D /2 \rfloor$ rounds old; we denote this number by $\uptau$ to reduce clutter.
	
	Under this worst-delay activation sequence, solving the decentralized online learning problem is equivalent to solving an OCO instance with the same losses and delay $\uptau$. Furthermore, the $k$-bit memory constraint affects identically the OCO instance. Note also that the definition of regret coincides under both settings. 
	
	The final claim is from \citet[Lemma 3]{zinkevich2009slow}; see also \citet[Proposition~~16]{hsieh2020multiagent} for a more complete argument that accounts for non-deterministic algorithms.
\end{proof}
All the ingredients to conclude are now available. Plugging in the lower bounds from Lemma~\ref{lem:det_lower_bound} with Theorem~\ref{thm:low_orab}, and Theorem~\ref{thm:lower_bound_stoch_opt} into Lemma~\ref{lem:delays_lower_bound}, we obtain the results of Theorems~\ref{th:lower_boundI} and Theorem~\ref{th:lower_boundII}, respectively.